\documentclass{article}

\usepackage[nonatbib,final]{nips_2016}

\usepackage[utf8]{inputenc} % allow utf-8 input
\usepackage[T1]{fontenc}    % use 8-bit T1 fonts
\usepackage{hyperref}       % hyperlinks
\usepackage{url}            % simple URL typesetting
\usepackage{booktabs}       % professional-quality tables
\usepackage{amsfonts}       % blackboard math symbols
\usepackage{nicefrac}       % compact symbols for 1/2, etc.
\usepackage{microtype}      % microtypography

\usepackage[numbers,square,comma,sort&compress]{natbib}
\usepackage{times}
\usepackage{calc}
\usepackage{mathtools}
\usepackage{amsmath, amssymb, amsthm}
\usepackage{color}
\usepackage{graphicx} % more modern
\usepackage{subfigure}
\usepackage{algorithm}
\usepackage{algorithmic}
\usepackage{booktabs}
\usepackage{enumitem}
\usepackage{bbm}
\usepackage{multirow}
\usepackage{kky}
\usepackage{nphmmDefns}
\usepackage{ptDefns}
\usepackage{booktabs}

\title{\Large{Learning HMMs with Nonparametric Emissions via \\
Spectral Decompositions of Continuous Matrices}}

\author{
  Kirthevasan~Kandasamy\thanks{Joint lead authors.} \\
  Carnegie Mellon University \\
  Pittsburgh, PA 15213 \\
  \texttt{kandasamy@cs.cmu.edu} \\
  \And
  Maruan~Al-Shedivat$^*$ \\
  Carnegie Mellon University \\
  Pittsburgh, PA 15213 \\
  \texttt{alshedivat@cs.cmu.edu} \\
  \And
  Eric~P.~Xing \\
  Carnegie Mellon University \\
  Pittsburgh, PA 15213 \\
  \texttt{epxing@cs.cmu.edu}
}

\begin{document}
\pdfoutput=1

\maketitle

% put all figures here.

\newcommand{\imtwowidth}{3.2in}
\newcommand{\imtwohspace}{-0.1in}
\newcommand{\imtextspace}{-0.15in}
\newcommand{\imcaptionspace}{-0.1in}

\newcommand{\insertQmCmFigure}{
\begin{figure*}
\subfigure{
  \includegraphics[width=\imtwowidth]{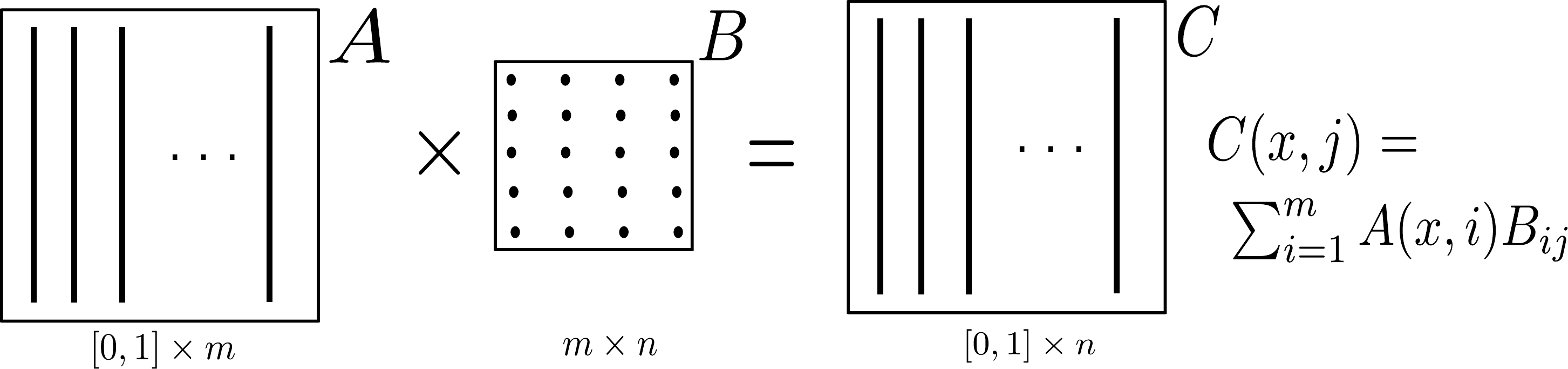} \hspace{\imtwohspace}
}
\subfigure{
  \includegraphics[width=\imtwowidth]{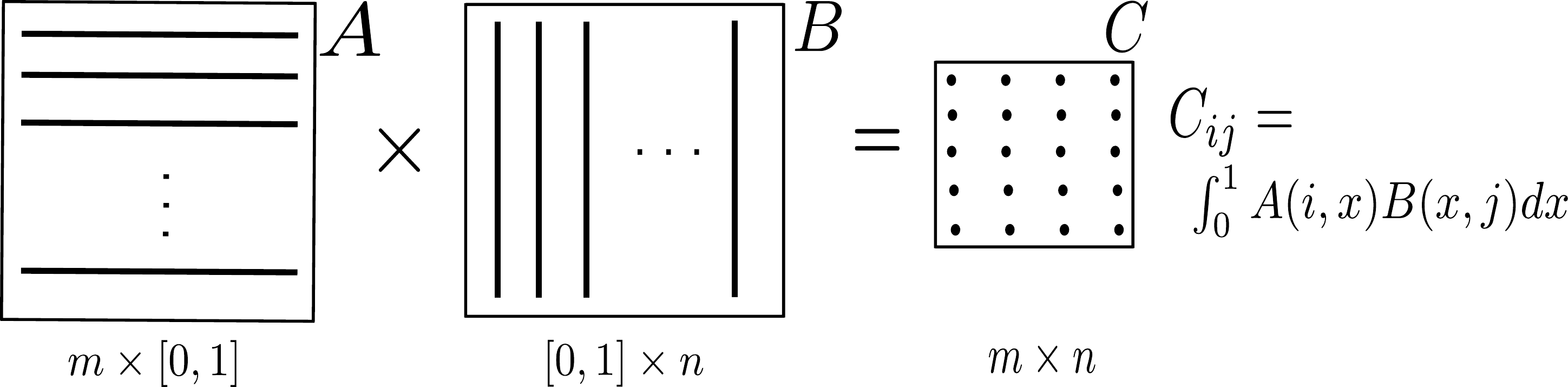} \hspace{\imtwohspace}
}
\vspace{\imcaptionspace}
\caption[]{
Illustration of qmatrix, cmatrix factorsiations.
\label{fig:QmCm}
}
\end{figure*}
}

\newcommand{\insertToyResults}{
  \begin{figure}[!t]
      \centering
      \includegraphics[width=0.325\columnwidth]{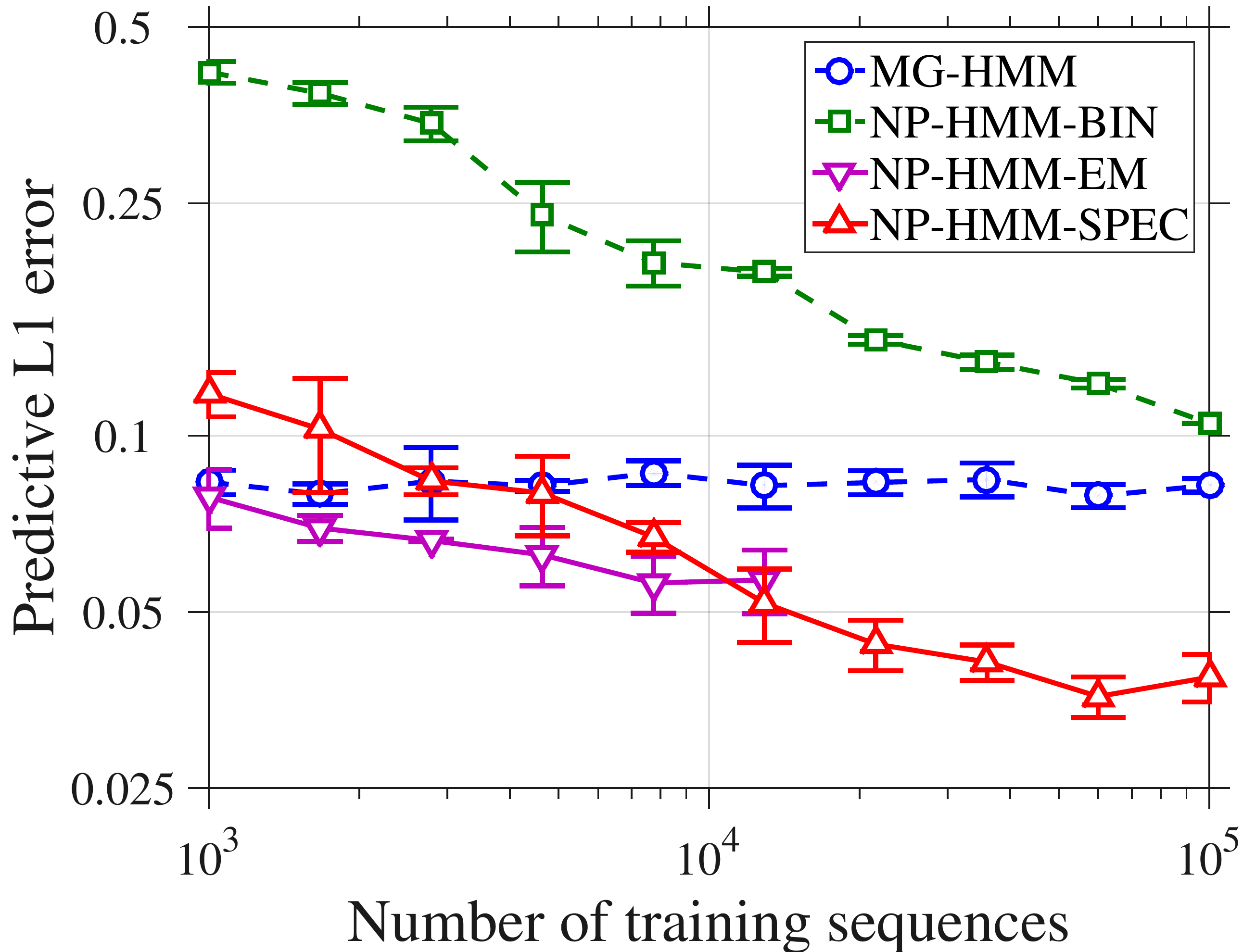}
      \includegraphics[width=0.325\columnwidth]{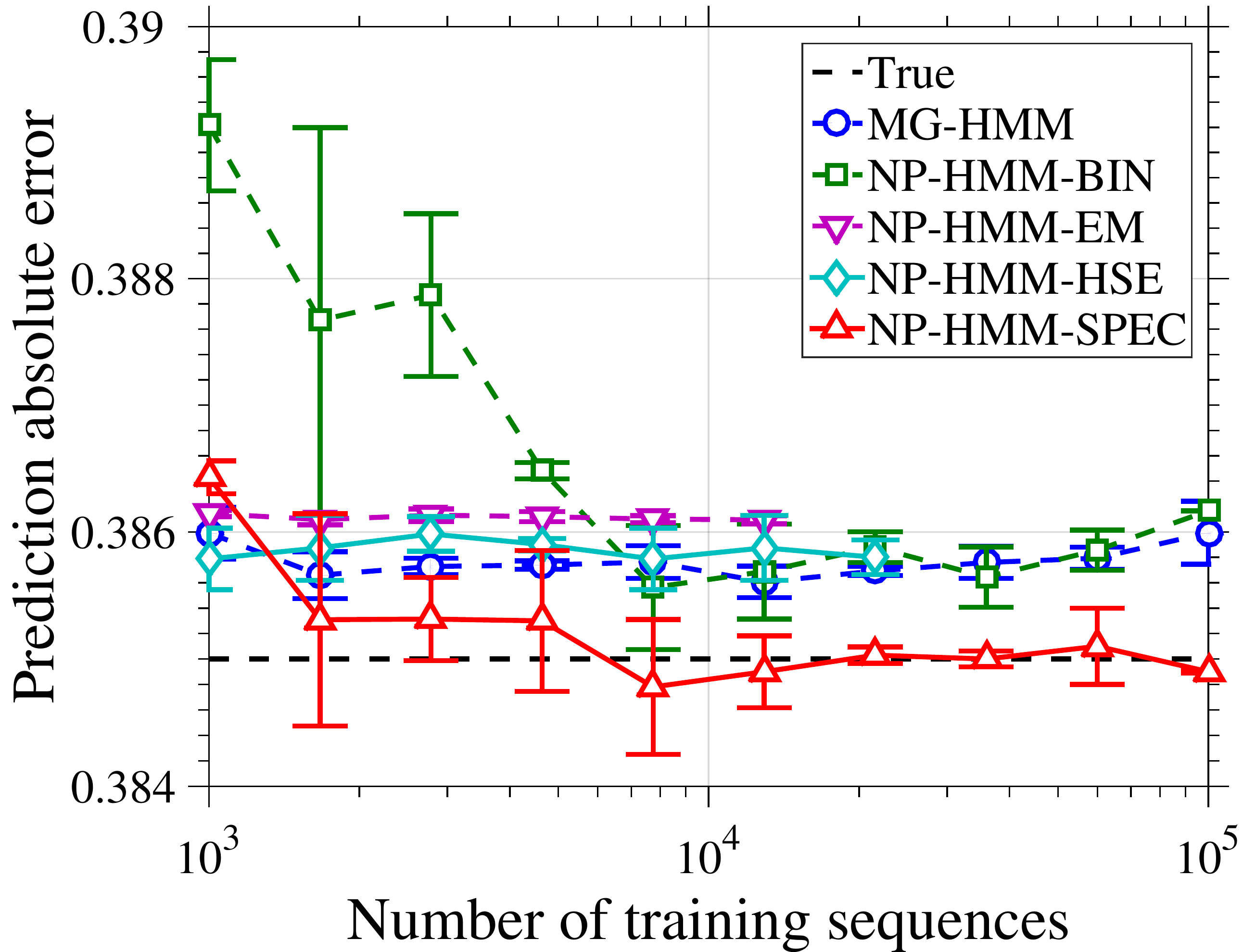}
      \includegraphics[width=0.325\columnwidth]{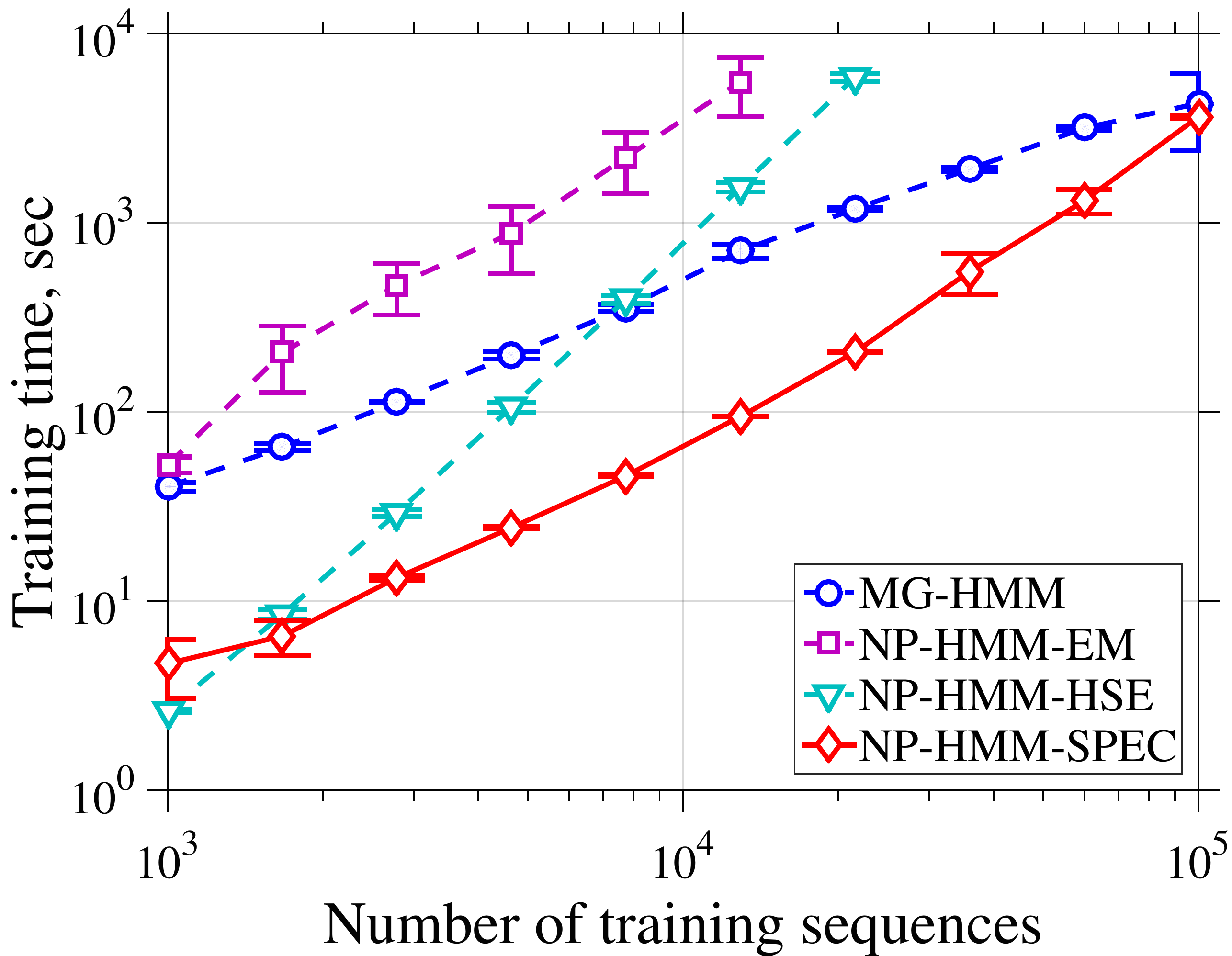}
      \includegraphics[width=0.325\columnwidth]{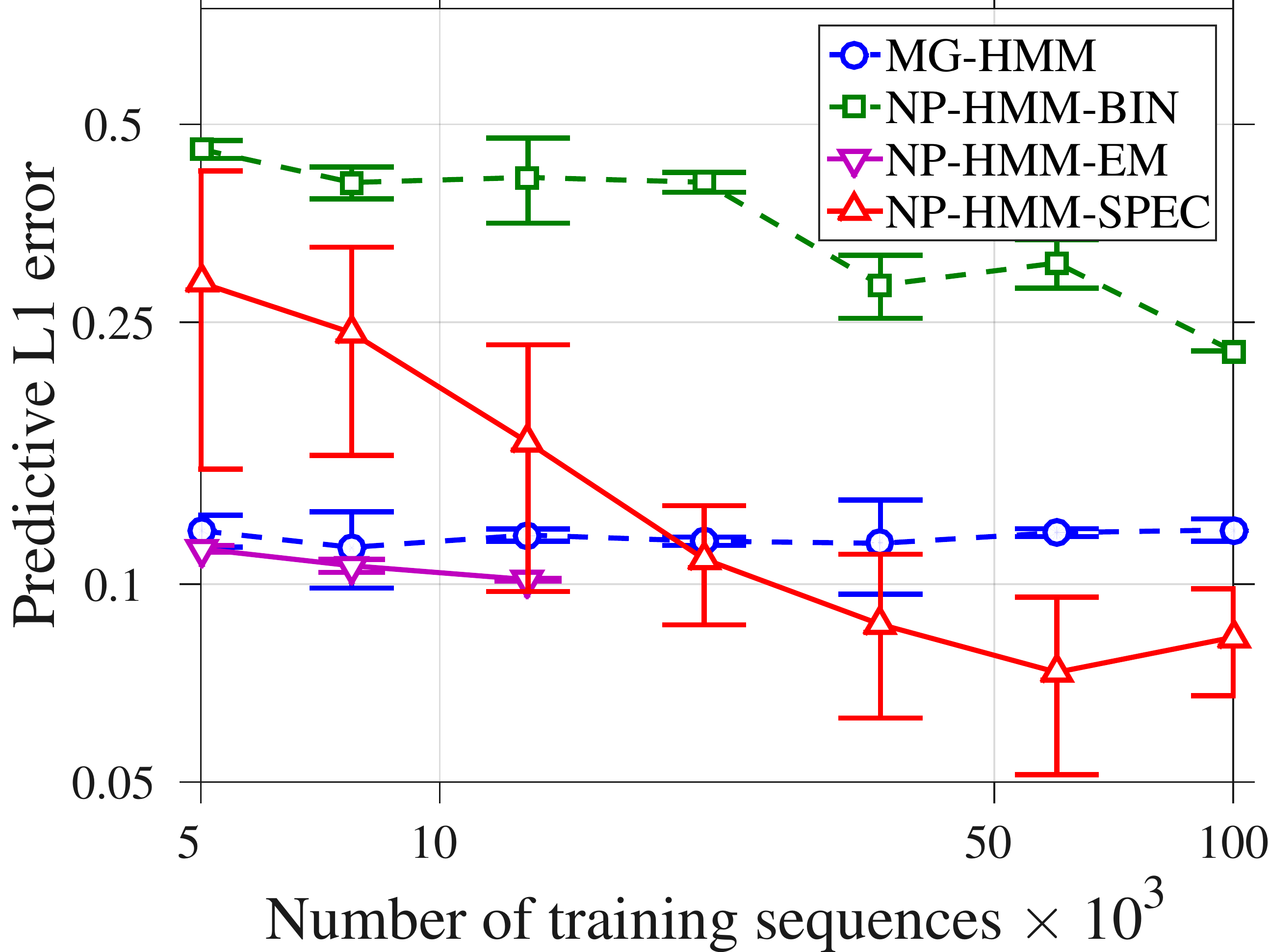}
      \includegraphics[width=0.325\columnwidth]{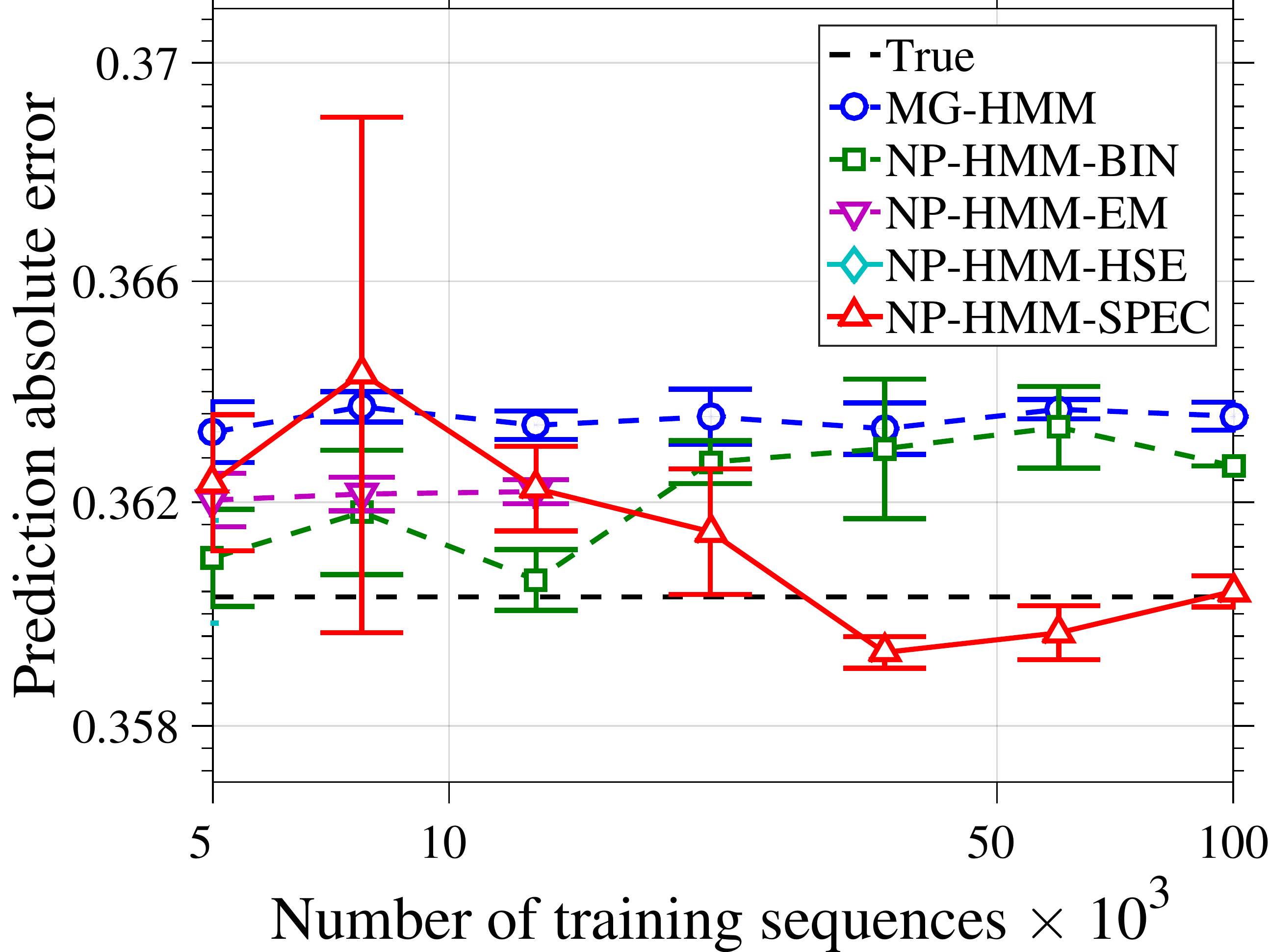}
      \includegraphics[width=0.325\columnwidth]{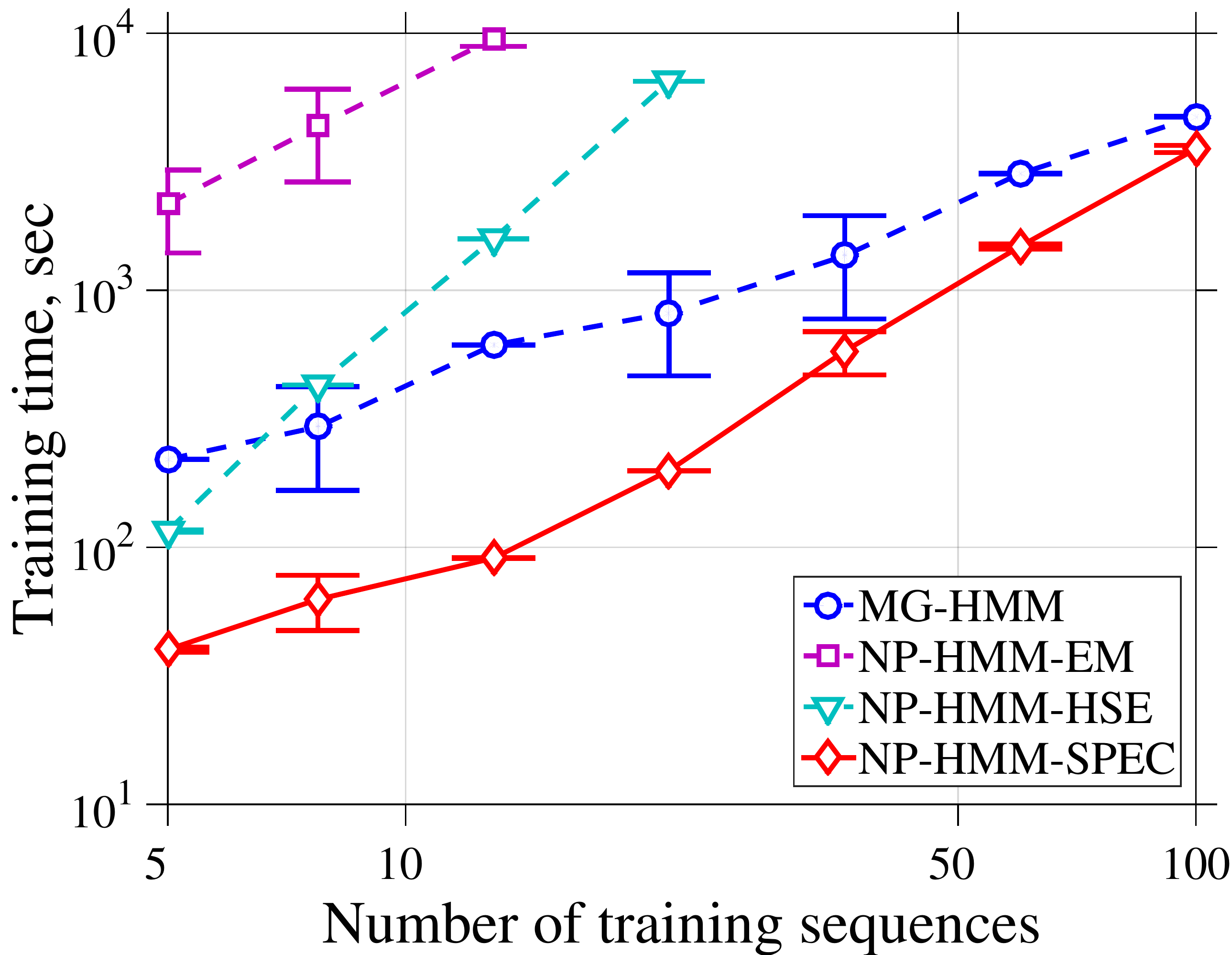}
      \vspace{-0.1in}
      \caption{\small 
      \label{fig:toyL1}
      The upper and lower panels correspond to $m=4$ $m=8$ respectively.
  All figures are in log-log scale and 
  the x-axis is the number of triples used for training. 
      \textbf{Left:} $L_1$ error between true conditional density
  $p(x_{6}| x_{1:5})$, and the estimate for each method.
      \textbf{Middle:} The absolute error between the true observation and a
  one-step-ahead prediction. The error of the true model is denoted by a black dashed
  line.
      \textbf{Right:} Training time.
}
      \vspace{-0.20in}
  \end{figure}
}

\newcommand{\insertToyPredictions}{
  \begin{figure}[!t]
      \centering
      \includegraphics[width=0.245\columnwidth]{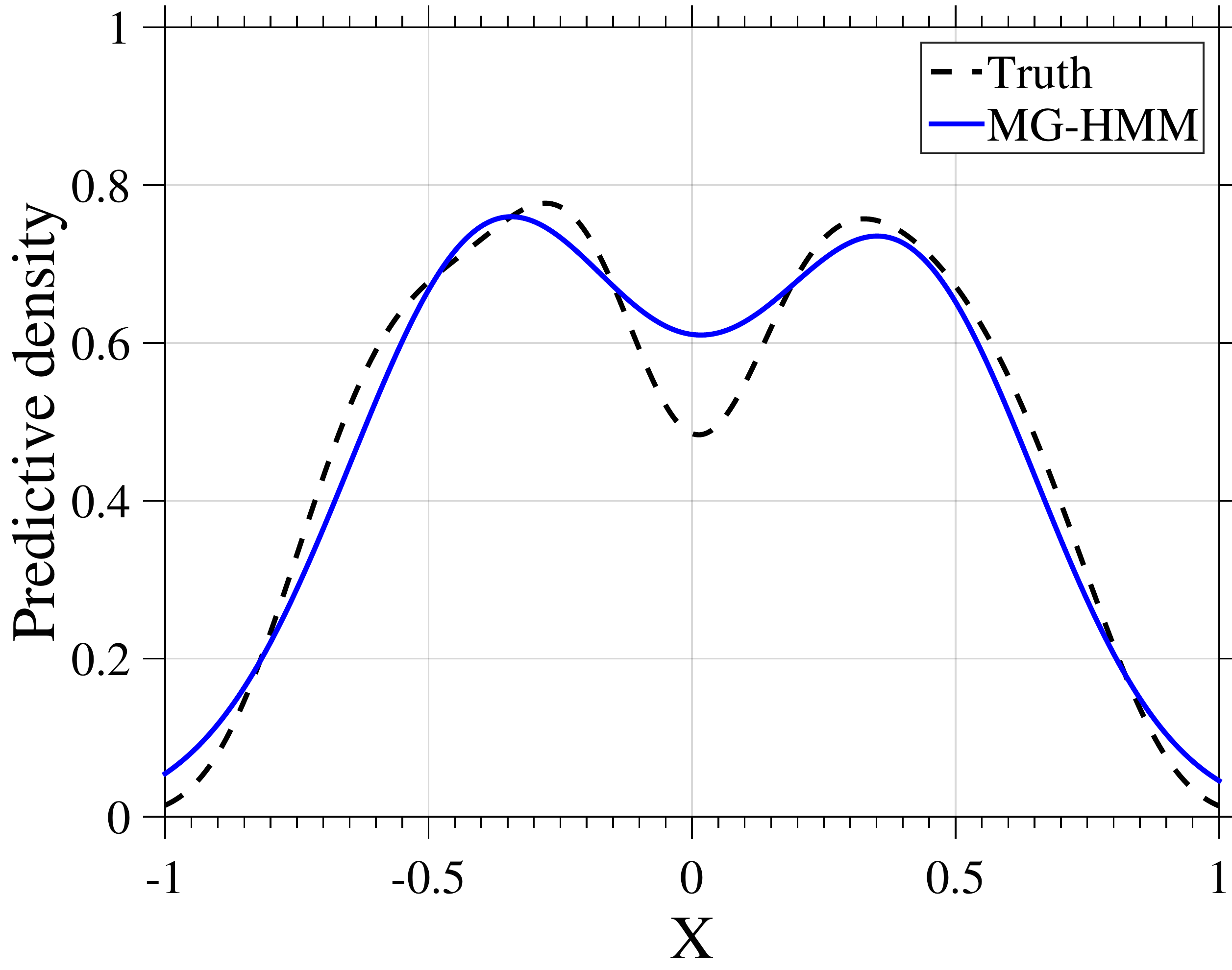}
      \includegraphics[width=0.245\columnwidth]{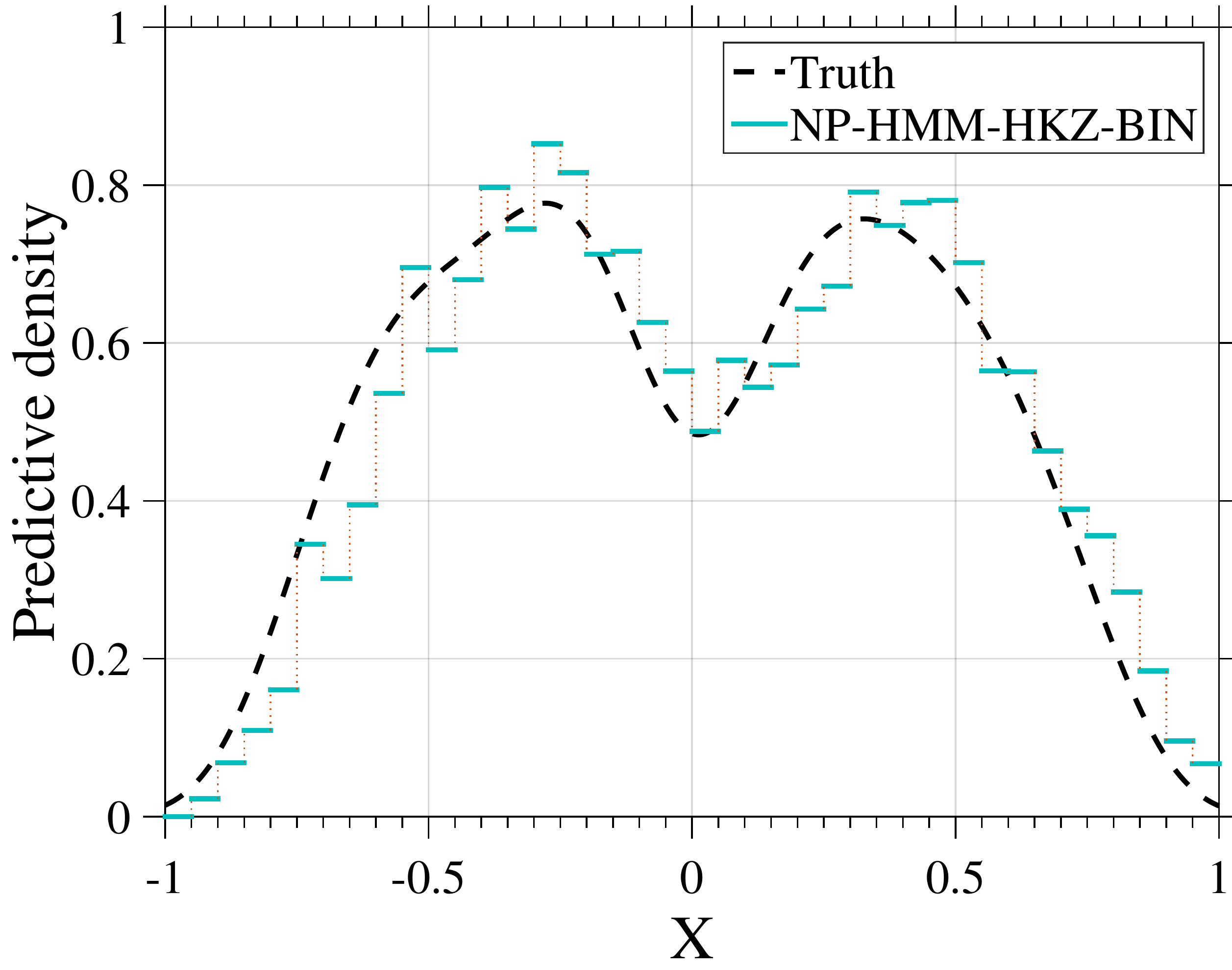}
      \includegraphics[width=0.245\columnwidth]{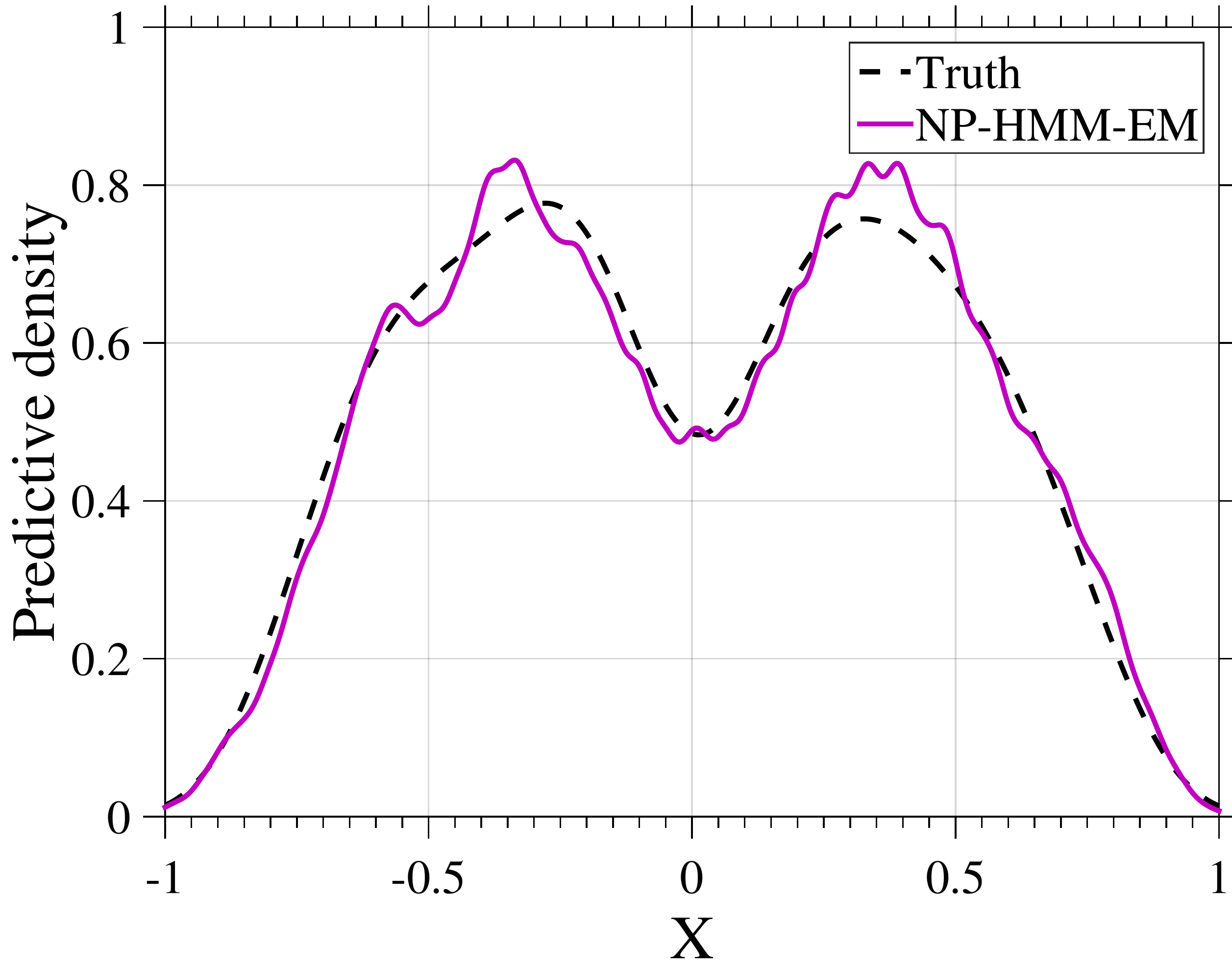}
      \includegraphics[width=0.245\columnwidth]{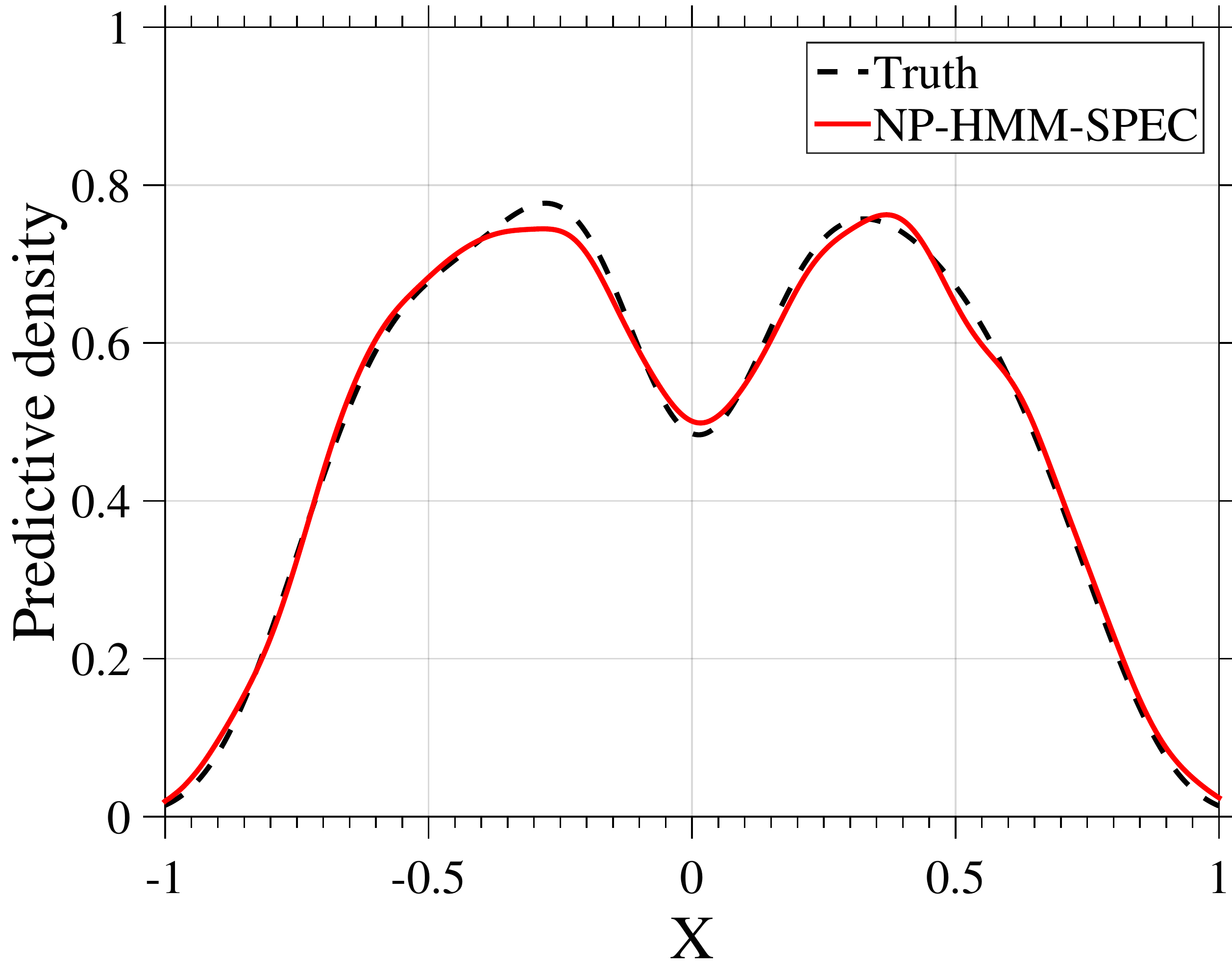}
      \vspace{-0.25in}
      \caption{\small
      \label{fig:toyPredProb}
      True and estimated one step ahead densities $p(x_4|x_{1:3})$ 
   for each model. Here $m=4$ and $N=10^4$.
%   The original HMM had 4 hidden states. Each model was trained on 10,000
%   examples.
    }
      \vspace{-0.05in}
  \end{figure}
}

\newcommand{\insertToyPDFs}{
  \begin{figure}
      \centering
      \includegraphics[width=0.45\columnwidth]{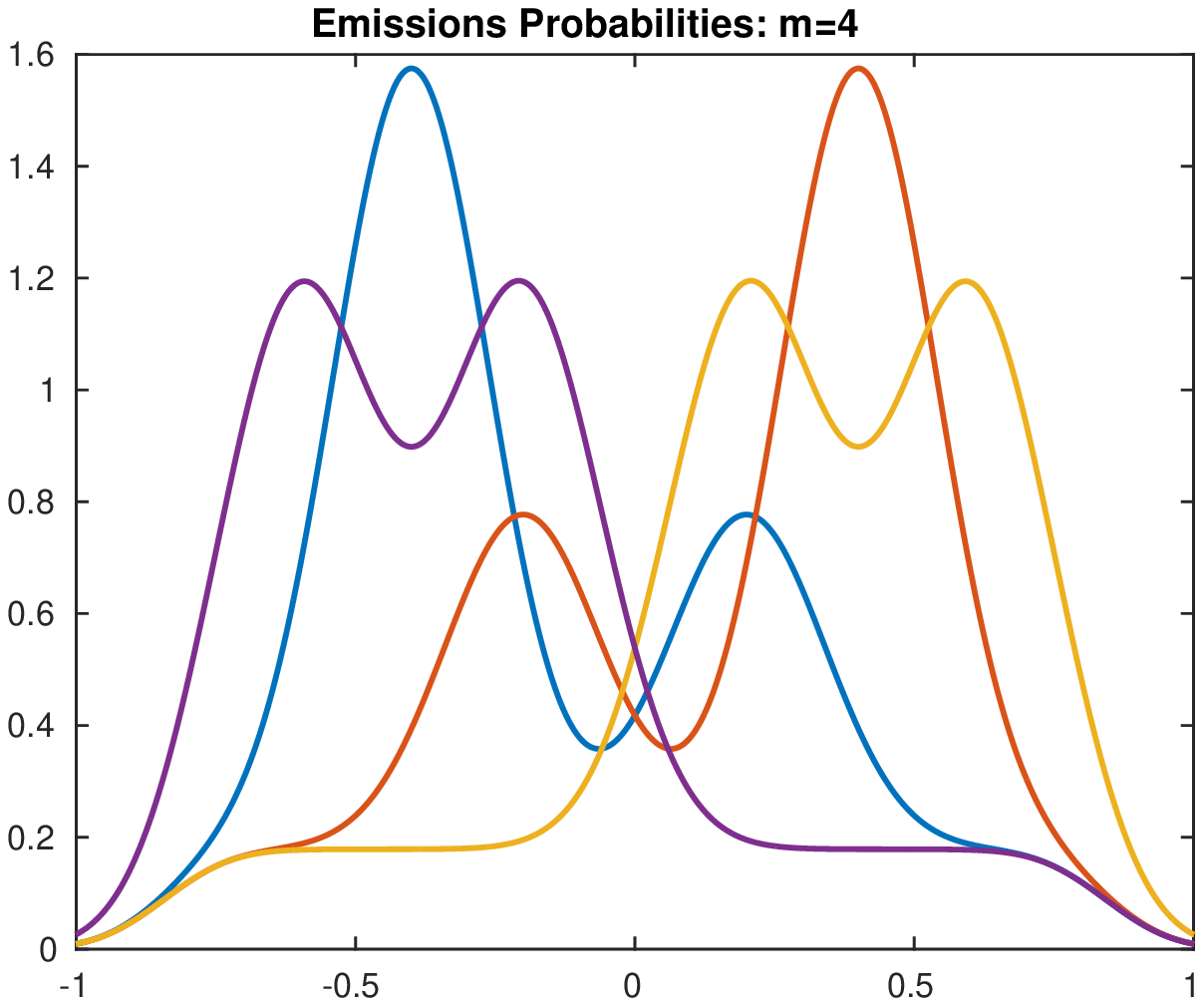} \hspace{0.2in} 
      \includegraphics[width=0.45\columnwidth]{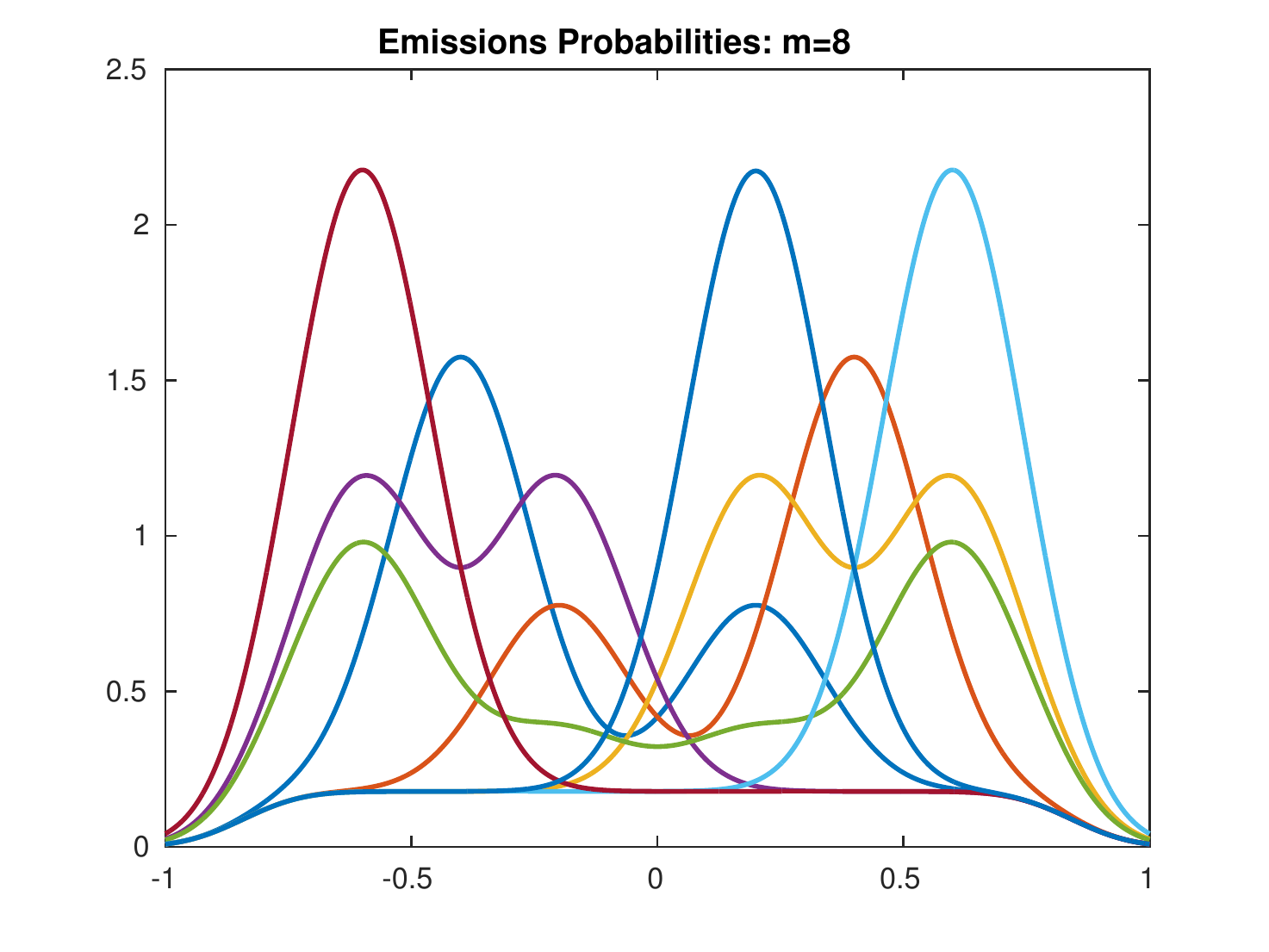}
  \caption{
  \label{fig:emissions}
  An illustration of the nonparametric emission probabilities used in our
  experiments.
  }
  \end{figure}
}

% Insert all tables here

\newcommand{\insertAlgo}{
\begin{algorithm}
\textbf{Input: } Data $\{\Xj = (\Xjone, \Xjtwo, \Xjthr)\}_{j=1}^N$, number of states
$m$.
\begin{itemize}[leftmargin=0.25in,itemsep=0ex]
\item Obtain estimates $\Pohat, \Ptohat, \Pttohat$ for $\Po, \Pto, \Ptto$ via kernel
density estimation~\eqref{eqn:kde}.
\item Compute the cmatrix SVD of $\Ptohat$. Let $\Uhat\in\RRm{[0,1]}{m}$ be the
first $m$ left singular vectors of $\Ptohat$.
\item Compute the parameters observable representation. Note that $\Bhat$ is a
$\RRm{m}{m}$ valued function.
\[
\bonehat = \Uhat^\top \Pohat, \hspace{0.5in}
\binfhat = (\Pto^\top \Uhat)^\dagger \Pohat , \hspace{0.5in}
\Bxhat = (\Uhat^\top \Ptxohat)(\Uhat^\top \Ptohat)^\dagger
\]
\end{itemize}
\caption{\hspace{0.05in}\npspechmm \label{alg:npspechmm}}
\vspace{-1ex}
\end{algorithm}
}

\newcommand{\insertRealResults}{
\begin{table}
\begin{center}
\small
\begin{tabular}{l|c|c|c|c}
\toprule
Dataset & \mghmm & \npbin & \nphse & \npspechmm \\
\midrule
Internet Traffic  & $0.143 \pm 0.001$ & $0.188 \pm 0.004$ & $0.0282 \pm 0.0003$
  & $\bf 0.016 \pm 0.0002$ \\
Laser Gen & $0.33 \pm 0.018$ & $0.31 \pm 0.017$ & $0.19\pm 0.012$ &
  $\bf 0.15 \pm 0.018$ \\
Patient Sleep & $0.330\pm 0.002$ & $0.38\pm 0.011$ & $\bf 0.197\pm 0.001$ &
  $0.225 \pm 0.001$\\
\bottomrule
\end{tabular}
\end{center}
\caption{\small
\label{tb:real}
The mean prediction error and the standard error on the $3$ real datasets.}
\vspace{-0.2in}
\end{table}
}

% !TEX root = ../paper.tex

\vspace{-2.5ex}
\begin{abstract}
Recently, there has been a surge of interest in using spectral methods for
estimating latent variable models. However, it is usually assumed that the
distribution of the observations conditioned on the latent variables is either
discrete or belongs to a parametric family. In this paper, we study the
estimation of an $m$-state hidden Markov model (HMM) with only smoothness
assumptions, such as H\"olderian conditions, on the emission densities. By
leveraging some recent advances in continuous linear algebra and numerical
analysis, we develop a computationally efficient spectral algorithm for
learning nonparametric HMMs. Our technique is based on computing
an SVD on nonparametric estimates of density functions by viewing them as
\emph{continuous matrices}.
We derive sample complexity bounds via concentration results for nonparametric
density estimation and novel perturbation theory results for continuous
matrices. We implement our method using Chebyshev polynomial
approximations. Our method is competitive with other baselines on
synthetic and real problems and is also very computationally efficient.
\end{abstract}
\vspace{-2.5ex}

% !TEX root = ../paper.tex

\section{Introduction}
\label{sec:intro}

Hidden Markov models (HMMs)~\citep{rabiner89hmm} are one of the most popular
statistical models for analyzing time series data in various application
domains such as speech recognition, medicine, and meteorology. In an HMM, a
discrete hidden state undergoes Markovian transitions from one of $m$ possible
states to another at each time step. If the hidden state at time $t$ is $\hht$,
we observe a random variable $\xt\in\Xcal$ drawn from an
\emph{emission} distribution, $\Obsj = \PP(\xt|\hht = j)$. In its most
basic form $\Xcal$ is a discrete set and $\Obsj$ are discrete distributions.
When dealing with
continuous observations, it is conventional to assume that the emissions $\Obsj$
belong to a parametric class of distributions, such as Gaussian.

% There are two key shortcomings with parametric models for the observations.
% First, the parametric assumption can be too restrictive for many practical
% applications. The learning procedure introduces a bias that we might not be
% able to recover from even with large datasets.
%% M: EM is not a shortcoming of just parametric models...
% The second is that estimation of the parameters is typically done using an
% Expectation Maximization (EM) procedure which is only guaranteed to find a
% local optimum. While there have been some recent advances to alleviate the
% latter issue~\cite{anandkumar12mom, anandkumar14tensor}, they only apply to
% certain classes of parametric models.

Recently, spectral methods for estimating parametric latent variable models
have gained immense popularity as a viable alternative to the Expectation
Maximisation (EM) procedure~\citep{hsu09hmm,anandkumar12mom,siddiqi10rrhmm}.
At a high level, these methods estimate higher order moments from the data and
recover the parameters via a series of matrix operations such as singular value
decompositions, matrix multiplications and pseudo-inverses of the moments. In
the case of discrete HMMs~\cite{hsu09hmm}, these moments correspond exactly to
the joint probabilities of the observations in the sequence.

Assuming parametric forms for the emission densities is often too restrictive
since real world distributions can be arbitrary.
Parametric models may introduce incongruous biases that cannot be
reduced even with large datasets. To address this problem, we study \emph{
nonparametric} HMMs only assuming some mild smoothness conditions on the emission
densities. We design a spectral algorithm for this setting. Our methods
leverage some recent advances in continuous linear algebra~
\cite{townsend15continuous,townsend14computing} which views two-dimensional
functions as continuous analogues of matrices.
Chebyshev polynomial approximations enable efficient computation of algebraic
operations on these continuous objects~\citep{driscoll2014chebfun,
townsend13chebfun}. Using these ideas, we extend existing spectral methods
for discrete HMMs to the continuous nonparametric setting.
Our main contributions are:

\begin{enumerate}[leftmargin=1.5em, itemsep=0ex]
\item We derive a spectral learning algorithm for HMMs with nonparametric
emission densities.
While the algorithm is similar to previous spectral methods for estimating models
with a finite number of parameters,  many of the ideas used to
generalise it to the nonparametric setting are novel, and, to the best of our
knowledge, have not been used before in the machine learning literature.
\item We establish sample complexity bounds for our method.
For this, we derive concentration results for nonparametric density
estimation and novel perturbation theory results for the aforementioned
continuous matrices. The perturbation results are new
and might be of independent interest.
\item
We implement our algorithm by approximating the density estimates via Chebyshev
polynomials which enables efficient computation of many of the continuous
matrix operations. Our method outperforms natural competitors in this setting
on synthetic and real data and is computationally more efficient
than most of them. Our Matlab code is available at
{\small\texttt{\href{https://github.com/alshedivat/nphmm}{github.com/alshedivat/nphmm}}}.
% {\small\url{https://github.com/alshedivat/nphmm}}.
\end{enumerate}

While we focus on HMMs in this exposition, we believe that the ideas presented
in this paper can be easily generalised to estimating other latent variable
models and predictive state representations~\cite{littman2001predictive} with
nonparametric observations using approaches developed
by~\citet{anandkumar12mom}.

\textbf{Related Work:}
Parametric HMMs are usually estimated using maximum likelihood principle via
EM techniques~\cite{dempster77em} such as the Baum-Welch
procedure~\citep{welch2003baumwelch}.
However, EM is a local search technique, and optimization of the likelihood may
be difficult.
% and therefore the solution is only guaranteed to converge to a local
% optimum. To obtain good performance in practice, several restarts may be
% required which can be computationally burdensome.
Hence, recent work on spectral methods has gained appeal. Our work builds
on~\citet{hsu09hmm} who showed that discrete HMMs can be learned efficiently,
under certain conditions. The key idea is that any
HMM can be completely characterised in terms of quantities that depend entirely
on the observations, called the \emph{observable representation}, which can be
estimated from data. \citet{siddiqi10rrhmm} show that the same algorithm works
under slightly more general assumptions. \citet{anandkumar12mom} proposed a
spectral algorithm for estimating more general latent variable models with
parametric observations via a moment matching technique.

That said, there has been little work on estimating latent variable models,
including HMMs, when the observations are \emph{nonparametric}. A commonly used
heuristic is the nonparametric EM~\citep{benaglia2009npem} which lacks
theoretical underpinnings.
% \footnote{Nonparametric EM has been originally designed for semi- and
% non-parametric mixture models, but it can be adapted to work with HMMs.}.
This should not be surprising because EM is a maximum likelihood procedure and,
for most nonparametric problems, the maximum likelihood estimate is
degenerate~\citep{wasserman06nonparametric}. In their work,
\citet{siddiqi10rrhmm} proposed a heuristic based on kernel smoothing,
with no theoretical justification, to modify the discrete algorithm for
continuous observations. Further, their procedure cannot be used to recover the
joint or conditional probabilities of a sequence which would be needed to
compute probabilities of events and other inference tasks.

\citet{song2010hilbert,song14multiview} developed an RKHS-based procedure for
estimating the Hilbert space embedding of an HMM. While they provide
theoretical guarantees, their bounds are in terms of the RKHS distance of the
true and estimated embeddings. This metric depends on the choice of the kernel
and it is not clear how it translates to a suitable distance measure on the
observation space such as an $\Lone$ or $\Ltwo$ distance.
While their method can be used for prediction and
pairwise testing, it cannot recover the joint and conditional densities.
On the contrary, our model provides guarantees in terms of the more
interpretable total variation distance and is able to recover the joint and conditional probabilities.

% The remainder of this paper is organised as follows. In Section~\ref
% {sec:cmatrix} we review some concepts in continuous linear algebra.
% Section~\ref{sec:hmmintro} details the Hidden Markov Model and the observable
% representation in the nonparametric setting. Section~\ref{sec:method} presents
% the algorithm and Section~\ref{sec:analysis} presents our theoretical results.
% Section~\ref{sec:implementation} describes our implementation and the
% experimental results are given in Section~\ref{sec:experiments}.

% !TEX root = ../paper.tex

\section{A Pint-sized Review of Continuous Linear Algebra}
\label{sec:cla}
\vspace{-0.05in}

We begin with a pint-sized review on continuous linear algebra which
treats functions as continuous analogues of matrices.
Appendix~\ref{app:cla} contains a quart-sized
review. Both sections are based on~\citep{townsend14computing,townsend15continuous}.
While these objects can be viewed as operators on Hilbert spaces
which have been studied extensively in the years, the above line of work
simplified and specialised the ideas to functions.

A \emph{matrix} $F\in\RR^{m\times n}$ is an $m\times n$ array of numbers where $F(i,j)$
denotes the entry in row $i$, column $j$.
 $m$ or $n$ could be (countably) infinite.
A \emph{column qmatrix} (quasi-matrix) $Q\in\RR^{[a,b]\times m}$ is a collection of
$m$ functions defined on $[a,b]$ where the row index is continuous and column index is
discrete.
Writing $Q = [q_1,\dots, q_m]$ where $q_j:[a,b]\rightarrow \RR$ is the $j$\ssth
function, $Q(y,j) = q_j(y)$ denotes the  value of the $j$\ssth function at $y\in[a,b]$.
$Q^\top\in\RR^{m\times [a,b]}$ denotes a row qmatrix with $Q^\top(j,y) = Q(y,j)$.
A \emph{cmatrix} (continuous-matrix) $C\in\RRm{[a,b]}{[c,d]}$ is a two
dimensional function where both row and column indices are continuous and
$C(y,x)$ is the value of the function at $(y,x) \in [a,b]\times[c,d]$.
$C^\top \in \RRm{[c,d]}{[a,b]}$ denotes its transpose with
$C^\top(x,y) = C(y,x)$.
Qmatrices and cmatrices permit all matrix multiplications with suitably defined
inner products. For example, if
$R\in \RRm{[c,d]}{m}$ and $C\in\RRm{[a,b]}{[c,d]}$, then
$CR = T\in\RRm{[a,b]}{m}$ where
$T(y,j) = \int_c^d C(y,s)R(s,j)\ud s$.

A cmatrix has a singular value decomposition (SVD).
If $C\in\RRm{[a,b]}{[c,d]}$, it decomposes as an infinite sum,
$C(y,x) = \sum_{j=1}^\infty \sigma_j u_j(y)v_j(x)$,
that converges in $\Ltwo$.
Here $\sigma_1\geq \sigma_2\geq \dots \geq 0$ are the singular values of $C$.
$\{u_j\}_{j\geq 1}$ and $\{v_j\}_{j\geq 1}$ are functions that form orthonormal
bases for $\Ltwo([a,b])$ and $\Ltwo([c,d])$, respectively.
% It is known that the SVD of a cmatrix exists uniquely with $\sigma_j \rightarrow 0$,
%  and continuous singular vectors
% (Theorem 3.2,~\citep{townsend14computing}).
% Further, if $C$ is Lipshcitz continuous w.r.t both
% variables then
% $\sum\sigma^2_j < \infty$ and
% the SVD is absolutely and uniformly convergent.
We can write the SVD as $C = U\Sigma V^\top$ by writing the singular vectors as
infinite qmatrices $U = [u_1, u_2\dots ], V=[v_1,v_2\dots]$, and
$\Sigma = \diag(\sigma_1,\sigma_2\dots )$.
If only $m < \infty$ first singular values are nonzero, we say that $C$ is of
rank $m$. The SVD of a qmatrix $Q\in\RRm{[a,b]}{m}$ is, $Q = U \Sigma V^\top$
where $U\in \RRm{[a,b]}{m}$ and $V\in\RRm{m}{m}$ have orthonormal columns
and $\Sigma = \diag(\sigma_1,\dots,\sigma_m)$ with
$\sigma_1\geq \dots \geq \sigma_m \geq 0$.
% The SVD of a qmatrix also exists uniquely (Theorem 4.1,~\citep{townsend14computing}).
The rank of a column qmatrix is the number of linearly independent columns
(i.e. functions) and is equal to the number of nonzero singular values.
Finally, as for the finite matrices, the pseudo inverse of the cmatrix $C$ is
$C^\dagger = V \Sigma^{-1}U^\top$ with
$\Sigma^{-1} = \diag( 1/\sigma_1, 1/\sigma_2, \dots )$.
The pseudo inverse of a qmatrix is defined similarly.

\section{Nonparametric HMMs and the Observable Representation}
% \label{sec:prelims}
\label{sec:hmmintro}

% We introduce nonparametric HMMs and the observable representation.
% This latter was studied for discrete observations
% by~\citet{hsu09hmm} and we follow their template for the nonparametric case.
\textbf{Notation: }
Throughout this manuscript, we will use $\PP$ to denote probabilities of
events while $p$ will denote probability density functions (pdf).
% $\one_m\in\RR^m$ denotes a vector of $1$'s.
An HMM characterises a probability distribution over a sequence of hidden
states $\{\hht\}_{t\geq 0}$ and observations $\{\xt\}_{t\geq 0}$.
At a given time step, the HMM can be in one of $m$ hidden states, i.e. $\hht
\in [m] = \{1,\dots,m\}$, and the observation is in some bounded continuous domain $\Xcal$. Without loss of generality, we take\footnote{
We discuss the case of higher dimensions in Section~\ref{sec:conclusion}.} $\Xcal=[0,1]$.
The nonparametric HMM will be completely characterised by the initial state
distribution $\pi\in\RRv{m}$, the state transition matrix $T\in\RRm{m}{m}$ and
the emission densities $\Obsj:\Xcal\rightarrow\RR, j\in[m]$. $\pi_i = \PP(
\hhone=i)$ is the probability that the HMM would be in state $i$ at the
first time step. The element $T(i,j) = \PP(\hhtpo=i|\hht=j)$ of $T$ gives the probability that a hidden state transitions from state $j$ to state $i$.
The emission function, $\Obsj:\Xcal\rightarrow \RR_+$, describes the
pdf of the observation conditioned on the hidden state $j$,
i.e. $\Obsj(s) = p(\xt = s|\hht = j)$.
Note that we have $\Obsj(x)>0,\,\forall x$ and $\int \Obsj(\cdot) = 1$ for all
$j\in[m]$. In this exposition, we denote the emission densities by the
qmatrix, $\Obs = [\Obsjj{1},\dots, \Obsjj{m}] \in \RR_+^{[0,1]\times m}$.

In addition, let $\Obsdiag(x) = \diag( \Obsjj{1}(x), \dots, \Obsjj{m}(x))$, and
$A(x) = T\Obsdiag(x)$. Let $\xotot = \{\xone, \dots, \xt\}$ be an ordered
sequence and $\xttoo = \{\xt, \dots, \xtt{1}\}$ denote its reverse. For
brevity, we will overload notation for $A$ for sequences and write $A(\xttoo) =
A(\xt)A(\xtt{t-1}) \dots A(\xtt{1})$.
It is well known~\citep{jaeger00observable,hsu09hmm} that the joint
probability density of the sequence $\xotot$ can be computed via
$p(\xotot) = \onem^\top A(\xttoo) \pi$.

\textbf{Key structural assumption:}
Previous work on estimating HMMs with continuous observations typically assumed
that the emissions, $\Obsj$, take a parametric form, e.g. Gaussian.
Unlike them, we only make mild nonparametric smoothness assumptions on
$\Obsj$. As we will see, to estimate the HMM well in this problem we will need to
estimate entire pdfs well.
For this reason, the nonparametric setting is significantly more
difficult than its parametric counterpart as the latter requires estimating
only a finite
number of parameters. When compared to the previous literature,
this is the crucial distinction and the main challenge in this work.
% when compared
% to previous literature.
%  we have $\Obs(y,j) =\Obsj(y) \geq 0, \forall
% y\in[0,1], j\in[m]$ and $\onezo^\top \Obs = \onem$.

\textbf{Observable Representation:}
The observable representation is a description of an HMM in terms of quantities
that depend on the observations~\citep{jaeger00observable}. This representation
is useful for two reasons: (i) it depends only on the observations and can
be directly estimated from the data; (ii) it can be used to compute joint and
conditional probabilities of sequences even without the knowledge of $T$ and
$\Obs$ and therefore can be used for inference and prediction.
First, we define the joint densities, $\Po, \Pto, \Ptto$:
\begin{align*}
\Po(t) = p(\xone=t), \quad
\Pto(s,t) = p(\xtwo=s, \xone=t), \quad
\Ptto(r,s,t) = p(\xthree=r,\xtwo=s,\xone=t),
\end{align*}
where $x_i$, $i=1,2,3$ denotes the observation at time $i$.
Denote $\Ptxo(r,t) = \Ptto(r, x, t)$ for all
$x$.
We will find it useful to view both $\Pto, \Ptxo \in
\RRm{[0,1]}{[0,1]}$ as cmatrices. We will also need
an additional qmatrix $U\in \RRm{[0,1]}{m}$ such that $U^\top \Obs \in
\RRm{m}{m}$ is invertible. Given one such  $U$, the observable representation of
an HMM is described by the parameters $\bone, \binf\in \RR^m$ and
$B:[0,1]\rightarrow \RRm{m}{m}$,
\begin{align*}
\bone = U^\top \Po, \qquad
\binf = (\Pto^\top U)^\dagger \Po, \qquad
B(x) = (U^\top \Ptxo)(U^\top \Pto)^\dagger
\numberthis \label{eqn:obsRepr}
\end{align*} As before, for a sequence,  $\xttoo = \{\xt, \dots, \xtt{1}\}$, we
define $B(\xttoo) = B(\xt)B(\xtt{t-1}) \dots B(\xtt{1})$. The following lemma
shows that the first $m$ left singular vectors of $\Pto$ are a natural choice
for $U$.
% \insertprespacing
\begin{lemma} Let $\pi > 0$, $T$ and $\Obs$ be of rank $m$ and $U$ be the
qmatrix composed of the first $m$ left singular vectors of $\Pto$. Then $U^\top
\Obs$ is invertible.
\label{lem:UPtO}
\end{lemma}
% \insertpostspacing
To compute the joint and conditional probabilities using the observable
representation, we maintain an \emph{internal state}, $\bt$, which is
updated as we see more observations. The internal state at time $t$ is
\begin{align}
\bt = \frac{B(\xtt{t-1:1}) \bone}{ \binf^\top B(\xtt{t-1:1})\bone}.
\label{eqn:internalState}
\end{align}
This definition of $\bt$ is consistent with $\bone$.
The following lemma establishes the relationship between the observable
representation and the internal states to the HMM parameters
and probabilities.
\insertprespacing
\begin{lemma}[Properties of the Observable Representation]
Let $\rank(T) = \rank(\Obs) = m$ and $U^\top \Obs$ be invertible.
Let $p(\xotot)$ denote the joint density of a sequence $\xotot$ and
$p(\xtt{t+1:t+t'}|\xotot)$ denote the conditional density of
$\xtt{t+1:t+t'}$ given $\xotot$ in a sequence $\xtt{1:t+t'}$.
Then the following are true.
\vspace{-0.15in}
\begin{center}
\parbox[t]{2.7in}{\raggedright%
\begin{enumerate}[topsep=0pt,itemsep=-2pt,leftmargin=13pt]
\item $\bone = \UtO \pi$
\item $\binf = \onem^\top \UtOinv$
\item $B(x) = (\UtO)A(x)\UtOinv\;\;\forall\,x\in[0,1]$.
\end{enumerate}
}%
\hspace{0.2in}
\parbox[t]{2.4in}{\raggedright%
\begin{enumerate}[topsep=0pt,itemsep=-2pt,leftmargin=13pt]
\setcounter{enumi}{3}
\item $\btpo = B(\xt)\bt / (\binf^\top B(\xt) \bt)$.
% \item $p(\xotot) = \onem^\top A(\xotot) \pi = \binf^\top B(\xotot) \bone$.
\item $p(\xotot)  = \binf^\top B(x_{t:1}) \bone$.
\item $p(\xtt{t+t':t+1}|\xotot) =  \binf^\top B(\xtt{t+t':t+1}) \bt$.
\end{enumerate}
}
\end{center}
\vspace{-0.07in}
\label{lem:obsProperties}
\end{lemma}
The last two claims of the Lemma~\ref{lem:obsProperties} show that we can use
the observable representation for computing the joint and conditional densities.
The proofs of Lemmas~\ref{lem:UPtO} and~\ref{lem:obsProperties} are similar to
the discrete case and mimic Lemmas 2, 3 \& 4 of~\citet{hsu09hmm}.

% \begin{proof}[\bf Proof of Lemmas~\ref{lem:UPtO},~\ref{lem:obsProperties}]
% % The proof is similar to the discrete case and mimics the steps in Lemmas 2,3 \& 4
% % of~\citet{hsu09hmm}.
% The proof is similar to the discrete case and mimics Lemmas 2,3 \& 4
%  of~\citet{hsu09hmm}.
% % The proof mimics the steps in Lemmas 2,3 \& 4
% % of~\citet{hsu09hmm} for the discrete case.
% \end{proof}

% \input{obshmm}  % cmatrix & hmmintro in one section
% !TEX root = ../paper.tex

\section{Spectral Learning of HMMs with Nonparametric Emissions}
\label{sec:method}

The high level idea of our algorithm, \npspechmm, is as follows.
First we will obtain density
estimates for $\Po,\Pto,\Ptto$ which will then be used to recover the observable
representation $\bone,\binf,B$ by plugging in the expressions in~\eqref{eqn:obsRepr}.
Lemma~\ref{lem:obsProperties} then gives us a way to estimate the joint and
conditional probability densities.
% We now present our algorithm.
For now, we will assume that we have $N$ \iid sequences
of triples $\{\Xj\}_{j=1}^N$ where $\Xj = (\Xjone, \Xjtwo, \Xjthr)$ are the
observations at the first three time steps.
We describe learning from longer sequences in Section~\ref{sec:implementation}.

\subsection{Kernel Density Estimation}
\vspace{-0.1in}

The first step is the estimation of the joint probabilities which requires a
nonparametric density estimate.
While there are several techniques~\cite{tsybakov08nonparametric}, we use
kernel density estimation (KDE) since it is easy to analyse and works  well in practice.
The KDE for $\Po$, $\Pto$, and $\Ptto$ take the form:
\begin{align*}
    &\Pohat(t) = \frac{1}{N}\sum_{j=1}^N \frac{1}{\ho}
        \kernel\left(\frac{t - \Xjone}{\ho}\right), \hspace{0.3in}
    \Ptohat(s, t) =
      \frac{1}{N}\sum_{j=1}^N \frac{1}{\hto^2}
        \kernel\left(\frac{s - \Xjtwo}{\hto}\right)
        \kernel\left(\frac{t - \Xjone}{\hto}\right), \\
    &\Pttohat(r, s, t) =
    \frac{1}{N}\sum_{j=1}^N \frac{1}{\htto^3}
      \kernel\left(\frac{r - \Xjthr}{\htto}\right)
      \kernel\left(\frac{s - \Xjtwo}{\htto}\right)
      \kernel\left(\frac{t - \Xjone}{\htto}\right).
    \numberthis \label{eqn:kde}
\end{align*}
Here $\kernel:[0,1]\rightarrow\RR$ is a symmetric function called a smoothing kernel
and satisfies (at the very least)  $\int_0^1 \kernel(s)\ud s = 1$,  $\int_0^1 s
\kernel(s) \ud s = 0$. The parameters $\ho, \hto, \htto$ are the bandwidths, and are
typically decreasing with $N$.
In practice they are usually chosen via cross-validation.

\subsection{The Spectral Algorithm}
\vspace{-0.15in}

\insertAlgo

The algorithm, given above in Algorithm~\ref{alg:npspechmm}, follows the roadmap set
out at the beginning of this section.
While the last two steps are similar
to the discrete HMM algorithm of~\citet{hsu09hmm}, the SVD, pseudoinverses and
multiplications are with q/c-matrices.
Once we  have the estimates $\bonehat$, $\binfhat$, and $\Bxhat$
the joint and predictive (conditional) densities can be estimated via (see
Lemma~\ref{lem:obsProperties}):
\begin{align*}
\phat(\xotot)  = \binfhat^\top \Bhat(x_{t:1}) \bonehat, \hspace{0.4in}
\phat(\xtt{t+t':t+1}|\xotot) =  \binfhat^\top \Bhat(\xtt{t+t':t+1}) \bthat.
\numberthis \label{eqn:predProb}
\end{align*}
Here $\bthat$ is the estimated internal state obtained by plugging in
$\bonehat,\binfhat,\Bhat$ in~\eqref{eqn:internalState}.
Theoretically, these estimates can be negative in which case they can be
truncated to $0$ without affecting the theoretical results in
Section~\ref{sec:analysis}.
However, in our experiments these estimates were never negative.

\subsection{Implementation Details}
\label{sec:implementation}

\textbf{C/Q-Matrix operations using Chebyshev polynomials: }
While our algorithm and analysis are conceptually well founded, the important
practical challenge lies in the efficient computation of the many aforementioned
operations on c/q-matrices.
Fortunately, some very recent advances in the numerical analysis literature,
specifically on computing with Chebyshev polynomials, have rendered the above
algorithm practical~\citep[Ch.3-4]{townsend14computing}. Due to the space
constraints, we provide only a summary.
Chebyshev polynomials is a family of orthogonal polynomials on
compact intervals, known to be an excellent approximator of one-dimensional
functions~\citep{fox68chebyshev,trefethen12approximation}. A recent line of
work~\citep{townsend13chebfun,townsend15continuous} has
extended the Chebyshev technology to two dimensional functions enabling the
mentioned operations and factorisations such as QR, LU and
SVD~\citep[Sections 4.6-4.8]{townsend14computing}
 of continuous matrices to be carried efficiently.
The density estimates $\Pohat,\Ptohat,\Pttohat$
are approximated by Chebyshev polynomials to within machine precision.
Our implementation makes use of the Chebfun library~\cite{driscoll2014chebfun}
which provides an efficient implementation for the operations on continuous and quasi matrices.

\textbf{Computation time: }
Representing the KDE estimates $\Pohat, \Ptohat, \Pttohat$ using Chebfun
was roughly linear in $N$ and is the brunt of the computational effort.
The bandwidths for the three KDE estimates are chosen via cross validation which
takes $\bigO(N^2)$ effort. However, in practice the cost was dominated by the
Chebyshev polynomial approximation.
% Once represented this way, the SVD and pseudoinverse are computed in constant time.
In our experiments we found that \npspechmms runs in linear time in practice
and was more efficient
than most alternatives.

\textbf{Training with longer sequences:}
When training with longer sequences we can use a sliding window of
length $3$ across the sequence to create the triples of observations needed for the
algorithm.
That is,  given $N$ samples each of length $\ell^{(j)}, j=1,\dots,N$, we create an
augmented dataset of triples
$ \{\,\{ (\Xj_{t}, \Xj_{t+1}, \Xj_{t+2}) \}_{t=1}^{\ell^{(j)}-2} \,\}_{j=1}^N$ and run
\npspechmms with the augmented data.
As is with conventional EM procedures,
this requires the additional assumption that the initial state is the stationary
distribution of the transition matrix $T$.

% !TEX root = ../paper.tex

\section{Analysis}
\label{sec:analysis}
\vspace{-0.1in}

We now state our assumptions and main theoretical results.
Following~\citep{hsu09hmm,siddiqi10rrhmm,song2010hilbert} we
assume \iid sequences of triples are used for training. With longer sequences,
the analysis should only be modified to account for the mixing of the latent
state Markov chain, which is inessential for the main intuitions.
We begin with the following regularity condition on the HMM.

\begin{assumption}
$\pi > 0$ element-wise.  $T\in\RRm{m}{m}$ and $\Obs\in\RRm{[0,1]}{m}$ are of
rank $m$.
% \toworkon{why $\pi > 0$?} % -> see discussion of Condition 1 in HKZ.
\label{asm:rankCondn}
\end{assumption}

The rank condition on $O$ means that emission pdfs are linearly independent.
If either $T$ or $O$ are rank deficient, then the learner may confuse state
outputs, which makes learning difficult\footnote{
\citet{siddiqi10rrhmm} show that the discrete spectral algorithm works under a
slightly more general setting. Similar
results hold for the nonparametric case too but will restrict ourselves to the
full rank setting for simplicity.}.
Next, while we make no parametric assumptions on the emissions,
some smoothness conditions are used to make density estimation tractable. We use the H\"older class, $\Hcal_1(\beta, L)$, which is standard in the
nonparametrics literature. For $\beta = 1$, this assumption reduces to $L$-Lipschitz continuity.

\begin{assumption}
All emission densities belong to the H\"older class, $\Hcal_1(\beta, L)$.
That is, they satisfy,
\begin{align*}
\textrm{for all}\; \alpha \leq \floor{\beta},\; j \in [m], \; s, t \in [0, 1]
\quad
\left|\frac{\ud^\alpha \Obsj(s)}{\ud s^\alpha} -
\frac{\ud^\alpha \Obsj(t)}{\ud t^\alpha}\right|
\,\leq \, L|s-t|^{\beta - |\alpha|} .
\end{align*}
Here $\floor{\beta}$ is the largest integer \textbf{strictly} less than $\beta$.
\label{asm:holder}
% \vspace{-2ex}
\end{assumption}

Under the above assumptions we bound the total variation distance between the true 
and the estimated
densities of a sequence, $\xotot$.
Let $\kappa(\Obs) = \sigma_1(\Obs)/\sigma_m(\Obs)$ denote the condition
number of the observation qmatrix.
The following theorem states our main result.

\begin{theorem}
Pick any sufficiently small $\epsilon > 0$ and a failure probability
$\delta \in (0,1)$.
Let $t\geq 1$. Assume that the HMM satisfies Assumptions~\ref{asm:rankCondn}
and~\ref{asm:holder} and the number of samples $N$ satisfies,
\begin{align*}
\frac{N}{\log(N)} \;\geq\;\;
C\, m^{1+\frac{3}{2\beta}}\frac{\kappa(\Obs)^{2+\frac{3}{\beta}}}{
\sigma_m(\Pto)^{4+\frac{4}{\beta}}}
  \left(\frac{t}{\epsilon}\right)^{2 + \frac{3}{\beta}}
\log\left(\frac{1}{\delta}\right)^{1+\frac{3}{2\beta}}.
\end{align*}
Then, with probability at least $1-\delta$, the estimated joint density for
a $t$-length sequence satisfies
$\int |p(\xotot) - \phat(\xotot)| \ud \xotot \leq \epsilon$.
Here, $C$ is a constant depending on $\beta$ and $L$ and
$\phat$ is from~\eqref{eqn:predProb}.
\label{thm:jointProbThm}
\end{theorem}

\textbf{Synopsis:}
Observe that the sample complexity depends critically on the conditioning of
$\Obs$ and $\Pto$. The closer they are to being singular, the more samples is
needed to distinguish different states and learn the HMM.
It is instructive to compare the results above with the discrete case result
of~\citet{hsu09hmm}, whose sample complexity bound\footnote{
\citet{hsu09hmm} provide a more refined bound but we use this form to simplify
the comparison.}
is $N \gtrsim m \frac{\kappa(\Obs)^{2}}{\sigma_m(\Pto)^4}
  \frac{t^2}{\epsilon^2}
\log\frac{1}{\delta}$.
Our bound is different in two regards.
First, the exponents are worsened by additional $\sim \frac{1}{\beta}$ terms.
This characterizes the difficulty of the problem in the nonparametric setting.
While we do not have any lower bounds, given the current understanding of the
difficulty of various nonparametric
tasks~\citep{birge95estimation,robins09quadraticvm,kandasamy15vonmises},
we think our bound might be unimprovable.
As the smoothness of the densities increases $\beta\rightarrow\infty$,
we approach the parametric sample complexity.
The second difference is the additional $\log(N)$ term on the left hand side.
This is due to the fact that we want the KDE to concentrate around its
expectation in $\Ltwo$ over $[0,1]$, instead of just
point-wise. It is not clear to us whether the $\log$ can be avoided.

To prove Theorem~\ref{thm:jointProbThm}, first we will derive
some perturbation theory results for c/q-matrices; we will need them to bound
the deviation of the singular values and vectors when we use $\Ptohat$ instead
of $\Pto$. 
Some of these perturbation theory results for continuous linear algebra are new and
might be of independent interest.
Next, we establish a concentration result for the kernel density estimator. 
% Finally, we combine these results to derive the sample complexity
% for estimating the joint density in the total variation metric.
% We start with the assumptions.

\subsection{Some Perturbation Theory Results for C/Q-matrices}
\label{sec:anPerturb}
% \vspace{-0.05in}

% The next step in the analysis is to bound the differences in the spectral
% properties of $\Ptohat$ with those of the true density $\Pto$.
The first result is an analog of Weyl's theorem which bounds the difference
in the singular values in terms of the operator norm of the perturbation.
Weyl's theorem has been studied for general operators~\citep{lee1998weyl}
and cmatrices~\cite{townsend14computing}. We have given one version in
Lemma~\ref{lem:weyl} of Appendix~\ref{app:PT}.
In addition to this, we will also need to bound the difference in the singular
vectors and the pseudo-inverses of the truth and the estimate. To our knowledge,
these results are not yet known.
To that end, we establish the following results.
Here $\sigma_k(A)$ denotes the $k$\ssth singular value of a c/q-matrix $A$.
% Our analysis follows the matrix proof found in~\citet{stewart90perturbation} although
% some care is needed with convergence when dealing with infinite matrices.
% \insertprespacing
\begin{lemma}[Simplified Wedin's Sine Theorem for Cmatrices]
Let $A,\Atilde, E\in \RRm{[0,1]}{[0,1]}$ where $\tilde{A} = A+E$ and $\rank(A) = m$.
Let $U,\Utilde\in\RRm{[a,b]}{m}$ be the first $m$ left singular vectors of $A$ and
$\Atilde$ respectively.
Then,  for all $x\in\RRv{m}$,
% $\|\Utilde^\top U x\|_2 \geq \|x\|_2\sqrt{1-\frac{2\|E\|^2_\frob}{\sigma_m(\Atilde)^2}}$.
$\|\Utilde^\top U x\|_2 \geq \|x\|_2\sqrt{1-2\|E\|^2_\Ltwo/\sigma_m(\Atilde)^2}$.
\label{lem:wedin}
\end{lemma}
% \insertprespacing
\begin{lemma}[Pseudo-inverse Theorem for Qmatrices]
Let $A,\Atilde, E\in \RRm{[a,b]}{m}$ and $\tilde{A} = A+E$. Then,
\[
% \|A^\dagger - \tilde{A}^\dagger\|_2 \leq
\sigma_1(A^\dagger - \tilde{A}^\dagger) \;\leq\;
% \frac{1+\sqrt{5}}{2}
% ((1+\sqrt{5})/2)
3
\, \max\{ \sigma_1(A^\dagger)^2,  \sigma_1(A^\dagger)^2 \}
\,\sigma_1(E).
\]
\label{lem:pinv}
\end{lemma}
\vspace{-0.2in}

\subsection{Concentration Bound for the Kernel Density Estimator}
\label{sec:kdeConcentration}

Next, we bound the error for kernel density estimation.
To obtain the best rates under H\"olderian assumptions on $\Obs$,
the kernels used in KDE need to be of \emph{order} $\beta$.
A $\beta$ order kernel satisfies,
\begin{align*}
\int_0^1 \kernel(s) \ud s = 1, \hspace{0.2in}
\int_0^1 s^\alpha\kernel(s) \ud s = 0, \textrm{for all $\alpha \leq
\floor{\beta} $},  \hspace{0.2in}
\int_0^1 s^\beta \kernel(s) \ud s \leq \infty. \quad
\numberthis \label{eqn:kernCondns}
\end{align*}
Such kernels can be constructed using Legendre
polynomials~\citep{tsybakov08nonparametric}.
Given $N$ \iid samples from a $d$ dimensional density $f$, where
$d\in\{1,2,3\}$ and $f \in\{\Po, \Pto, \Ptto\}$,
for appropriate choices of the bandwidths $\ho,\hto,\htto$, the KDE
$\fhat\in\{\Pohat,\Ptohat,\Pttohat\}$
concentrates around $f$.
Informally, we show
\begin{align}
\PP\left( \|\fhat - f\|_\Ltwo > \varepsilon \right)
\;\lesssim\; \exp\left( - \log(N)^{\frac{d}{2\beta+d}}
  N^{\frac{2\beta}{2\beta+d}} \varepsilon^2\right).
\label{eqn:concKDE}
\end{align}
for all sufficiently small $\varepsilon$ and $N/\logN \gtrsim
\varepsilon^{-\tpdob}$.
Here $\lesssim,\gtrsim$ denote inequalities ignoring constants.
See Appendix~\ref{sec:concKDE} for a formal statement.
% While this result is not fundamentally new (see~\citep{liu11forest,gine02rates}
% for some similar results), we derive the above form
% specifically needed for our analysis.
Note that
when the observations are either discrete or parametric, it is possible to estimate
the distribution  using $O(1/\varepsilon^2)$ samples to achieve
$\varepsilon$ error in a suitable metric, say, using the maximum likelihood estimate.
However, the nonparametric setting is inherently more difficult and therefore the
rate of convergence is slower.
This slow convergence is also observed
in similar concentration bounds for the KDE~\citep{liu11forest,gine02rates}.
% Given that $N^{\frac{-2\beta}{2\beta + d}}$ is the
% minimax rate for estimating a nonparametric density in $\Ltwo$
% error~\citep{wasserman06nonparametric}, we believe the
% above concentration result is unimprovable up to $\polylog$ factors.
% In our analysis use Lemma~\ref{lem:concKDE} to control the errors
% $\|\Pohat - \Po\|_{\Ltwo}$, $\|\Ptohat - \Pto\|_{\Ltwo}$,
% and $\|\Pttohat - \Ptto\|_{\Ltwo}$.

\textbf{A note on the Proofs:}
For Lemmas~\ref{lem:wedin},~\ref{lem:pinv} we follow the matrix proof
in~\citet{stewart90perturbation} and derive several intermediate results for
c/q-matrices in the process.
The main challenge is that several properties for matrices, e.g. the CS and Schur decompositions, are not known for c/q-matrices. In addition, dealing
with various notions of convergences with these infinite objects can be finicky.
The main challenge with the KDE concentration result is that we want an $\Ltwo$
bound -- so usual techniques (such as
McDiarmid's~\citep{tsybakov08nonparametric,wasserman06nonparametric})
do not apply.
We use a technical lemma from~\citet{gine02rates} which allows us to
bound the $\Ltwo$ error in terms of the VC characteristics of the class of functions
induced by an \iid sum of the kernel.
The proof of theorem~\ref{thm:jointProbThm} just mimics the discrete case analysis
of~\citet{hsu09hmm}.
While, some care is needed (e.g. $\|x\|_\Ltwo
\leq \|x\|_\Lone$ does not hold for functional norms) the key ideas carry through once
we apply
Lemmas~\ref{lem:weyl},~\ref{lem:wedin},~\ref{lem:pinv} and~\eqref{eqn:concKDE}.
A more refined bound on $N$ that is tighter in $\polylog(N)$ terms is
possible -- see Corollary~\ref{cor:KDE} and equation~\ref{eqn:epslnBounds}
in the appendix.

% \input{implementation}
% !TEX root = ../paper.tex

\section{Experiments}
\label{sec:experiments}
\vspace{-0.1in}

%%%%%%%%%%%%%%%%%%%%%%%%%%%%%%%%%%%%%%%%%%%%%%%%%%%%%%%%%%%%%%%%%%%%%
\insertToyResults
% Toy Results
%%%%%%%%%%%%%%%%%%%%%%%%%%%%%%%%%%%%%%%%%%%%%%%%%%%%%%%%%%%%%%%%%%%%%

We compare \npspechmms to the following.
\mghmm: An HMM trained using EM with the emissions modeled as a
mixture of Gaussians. We tried $2,4$ and $8$ mixtures and report the best
result.
\npbin: A naive baseline where we bin the space into $n$
intervals and use the discrete spectral algorithm~\citep{hsu09hmm} with $n$
states.
We tried several values for $n$ and report the best.
\npem: The Nonparametric EM heuristic of~\citep{benaglia2009npem}.
\nphse: The Hilbert space embedding method of~\citep{song2010hilbert}.

\textbf{Synthetic Datasets:}
We first performed a series of experiments on synthetic data where the true
distribution is known.
The goal is to evaluate the estimated models against the \emph{true} model.
We generated triples from two HMMs with $m=4$ and $m=8$ states and
nonparametric emissions.
The details of the set up are given in Appendix~\ref{app:experiments}.
Fig.~\ref{fig:toyL1} presents the results.

%%%%%%%%%%%%%%%%%%%%%%%%%%%%%%%%%%%%%%%%%%%%%%%%%%%%%%%%%%%%%%%%%%%%%%%%%%%%%%
\insertToyPredictions
%%%%%%%%%%%%%%%%%%%%%%%%%%%%%%%%%%%%%%%%%%%%%%%%%%%%%%%%%%%%%%%%%%%%%%%%%%%%%%

First we compare the methods on estimating the one step ahead conditional density
$p(x_6|x_{1:5})$. We report the $\Lone$ error between the true and estimated models.
In Fig.~\ref{fig:toyPredProb} we visualise the estimated one step ahead conditional
densities. \npspechmms outperforms all methods on this metric.
Next, we compare the methods on the prediction performance. That is, we sample
sequences of length $6$ and test how well a learned model can predict $x_6$
conditioned on $x_{1:5}$.
When comparing on squared error, the best predictor is the mean of the distribution.
For all methods we use the mean of $\phat(x_6|x_{1:5})$ except for \nphses for which
we used the mode since the mean cannot be computed. No method can do better than the
true model (shown via the dotted line) in expectation. \npspechmms achieves the
performance of the true model with large datasets.
Finally, we compare the training times of all methods. \npspechmms is
orders of magnitude faster than \nphses and \npem.

Note that the error of \mghmm -- a parametric model -- stops decreasing even with
large data. This is due to the bias introduced by the parametric assumption.
We do not train \npems for longer sequences because it is too slow.
A limitation of the \nphses method is that it
cannot recover conditional probabilities -- so we exclude it from that experiment.
We exclude \npbins  from the time comparison because it was much faster than all other
methods.
We could not include the method of~\citep{siddiqi10rrhmm} in our comparisons
since their code was not
available and their method isn't straightforward to implement. Further, their method
cannot compute joint/predictive probabilities.

\textbf{Real Datasets:}
We compare all the above methods (except \npems which was too slow) on
prediction error  on $3$ real datasets: internet traffic~\citep{paxson1995wide}, laser
generation~\citep{hubner1989dimensions} and sleep data~\citep{sleep}.
Each model was trained using a training sequence and then the predictions were
computed on a test sequence. The details on these datasets are in
Appendix~\ref{app:experiments}.
For all methods we used the mode of the conditional distribution $p(\xtt{t+1}|\xotot)$
as the prediction as it performed better.
For \npspechmm, \nphse,\npbins we follow the procedure outlined in
Section~\ref{sec:implementation} to create triples and train with the triples.
In Table~\ref{tb:real} we report the mean prediction error and the standard error.
\nphses and \npspechmms perform better than the other two methods.
However, \npspechmms was faster to train (and has other
attractive properties) when compared to \nphse.
% In addition, we reiterate that \npspechmms has several desirable characteristics when
% compared to \nphse: we
\insertRealResults

% !TEX root = ../paper.tex

\vspace{-0.1in}
\section{Conclusion}
\label{sec:conclusion}
\vspace{-0.1in}

We proposed and studied a method for estimating the observable representation
of a Hidden Markov Model whose emission probabilities are smooth nonparametric
densities.
We derive a bound on the sample complexity for our method.
% Sample complexity bound for the method has been derived.
% We begin with a kernel
% density estimate of the observations and then compute a cmatrix-svd of the
% joint probabilities. We estimate the joint and conditional probabilities by
% obtaining the observable representation of the HMM.
While our algorithm is similar to existing methods for discrete models,
many of the ideas that generalise it to the nonparametric setting
are new. In comparison to other methods, the proposed approach has some
desirable characteristics: we can recover the joint/conditional densities,
our theoretical results are in terms of more interpretable metrics,
the method outperforms baselines and is orders of magnitude
faster to train.

In this exposition only focused on one dimensional observations.
The multidimensional case is handled by extending the above ideas and
technology to multivariate functions.
% For instance, a $d$-dimensional cmatrix
% SVD will take the form $C(y,x) = \sum_j \sigma_j u_j(y)v_j(x)$ where
% $x,y\in\RR^d$ and $\{u_j\}_{j\geq 1}$, $\{v_j\}_{j\geq 1}$ form orthonormal
% bases for $[0,1]^d$.
Our algorithm and the analysis carry through
to the $d$-dimensional setting, \emph{mutatis mutandis}. The concern however,
is more practical. While we have the technology to perform various
c/q-matrix operations for $d=1$ using Chebyshev
polynomials, this is not \emph{yet} the case for $d>1$.
Developing efficient procedures for these operations in the high
dimensional settings is a challenge for the numerical analysis
community and is beyond the scope of this paper.
That said, some recent advances in this direction are
promising~\citep{hashemi16cheb3,townsend13chebfun}.

While our method has focused on HMMs, the ideas in this paper apply for a much
broader class of problems.
Recent advances in spectral methods for estimating parametric predictive state
representations~\cite{singh04psr}, mixture models~\cite{anandkumar12mom} and
other latent variable models~\cite{anandkumar14tensor}
can be generalised to the nonparamatric setting using our ideas. Going
forward, we wish to focus on such models.

% !TEX root = ../paper.tex

\subsection*{Acknowledgements}
\vspace{-0.1in}
The authors would like to thank Alex Townsend, Arthur Gretton,
and Ahmed Hefny for the helpful discussions.


{\small
\renewcommand{\bibsection}{\section*{References\vspace{-0.1em}} }
\setlength{\bibsep}{1.1pt}
\bibliography{kky,nphmm}

\begin{thebibliography}{32}
\providecommand{\natexlab}[1]{#1}
\providecommand{\url}[1]{\texttt{#1}}
\expandafter\ifx\csname urlstyle\endcsname\relax
  \providecommand{\doi}[1]{doi: #1}\else
  \providecommand{\doi}{doi: \begingroup \urlstyle{rm}\Url}\fi

\bibitem[Rabiner(1989)]{rabiner89hmm}
Lawrence~R. Rabiner.
\newblock {A Tutorial on Hidden Markov Models and Selected Applications in
  Speech Recognition}.
\newblock In \emph{Proceedings of the IEEE}, 1989.

\bibitem[Hsu et~al.(2009)Hsu, Kakade, and Zhang]{hsu09hmm}
Daniel~J. Hsu, Sham~M. Kakade, and Tong Zhang.
\newblock {A Spectral Algorithm for Learning Hidden Markov Models.}
\newblock In \emph{COLT}, 2009.

\bibitem[Anandkumar et~al.(2012)Anandkumar, Hsu, and Kakade]{anandkumar12mom}
Animashree Anandkumar, Daniel Hsu, and Sham~M Kakade.
\newblock {A Method of Moments for Mixture Models and Hidden Markov Models}.
\newblock \emph{arXiv preprint arXiv:1203.0683}, 2012.

\bibitem[Siddiqi et~al.(2010)Siddiqi, Boots, and Gordon]{siddiqi10rrhmm}
Sajid~M. Siddiqi, Byron Boots, and Geoffrey~J. Gordon.
\newblock {Reduced-Rank Hidden Markov Models}.
\newblock In \emph{AISTATS}, 2010.

\bibitem[Townsend and Trefethen(2015)]{townsend15continuous}
Alex Townsend and Lloyd~N Trefethen.
\newblock Continuous analogues of matrix factorizations.
\newblock In \emph{Proc. R. Soc. A}, 2015.

\bibitem[Townsend(2014)]{townsend14computing}
Alex Townsend.
\newblock \emph{{Computing with Functions in Two Dimensions}}.
\newblock PhD thesis, University of Oxford, 2014.

\bibitem[Driscoll et~al.(2014)Driscoll, Hale, and
  Trefethen]{driscoll2014chebfun}
Tobin~A Driscoll, Nicholas Hale, and Lloyd~N Trefethen.
\newblock Chebfun guide.
\newblock \emph{Pafnuty Publ}, 2014.

\bibitem[{Townsend, Alex and Trefethen, Lloyd N.}(2013)]{townsend13chebfun}
{Townsend, Alex and Trefethen, Lloyd N.}
\newblock An extension of chebfun to two dimensions.
\newblock \emph{SIAM J. Scientific Computing}, 2013.

\bibitem[Littman et~al.(2001)Littman, Sutton, and Singh]{littman2001predictive}
Michael~L Littman, Richard~S Sutton, and Satinder~P Singh.
\newblock Predictive representations of state.
\newblock In \emph{NIPS}, volume~14, pages 1555--1561, 2001.

\bibitem[Dempster et~al.(1977)Dempster, Laird, and Rubin]{dempster77em}
A.~P. Dempster, N.~M. Laird, and D.~B. Rubin.
\newblock {Maximum likelihood from incomplete data via the EM algorithm}.
\newblock \emph{JOURNAL OF THE ROYAL STATISTICAL SOCIETY, SERIES B}, 1977.

\bibitem[Welch(2003)]{welch2003baumwelch}
Lloyd~R Welch.
\newblock {Hidden Markov models and the Baum-Welch algorithm}.
\newblock \emph{IEEE Information Theory Society Newsletter}, 2003.

\bibitem[Benaglia et~al.(2009)Benaglia, Chauveau, and Hunter]{benaglia2009npem}
Tatiana Benaglia, Didier Chauveau, and David~R Hunter.
\newblock {An EM-like algorithm for semi-and nonparametric estimation in
  multivariate mixtures}.
\newblock \emph{Journal of Computational and Graphical Statistics}, 2009.

\bibitem[Wasserman(2006)]{wasserman06nonparametric}
Larry Wasserman.
\newblock \emph{{All of Nonparametric Statistics}}.
\newblock Springer-Verlag NY, 2006.

\bibitem[Song et~al.(2010)Song, Boots, Siddiqi, Gordon, and
  Smola]{song2010hilbert}
Le~Song, Byron Boots, Sajid~M Siddiqi, Geoffrey~J Gordon, and Alex Smola.
\newblock Hilbert space embeddings of hidden markov models.
\newblock In \emph{ICML}, 2010.

\bibitem[Song et~al.(2014)Song, Anandkumar, Dai, and Xie]{song14multiview}
Le~Song, Animashree Anandkumar, Bo~Dai, and Bo~Xie.
\newblock {Nonparametric Estimation of Multi-View Latent Variable Models}.
\newblock In \emph{Proceedings of the 31st International Conference on Machine
  Learning (ICML-14)}, pages 640--648, 2014.

\bibitem[Jaeger(2000)]{jaeger00observable}
Herbert Jaeger.
\newblock {Observable operator models for discrete stochastic time series}.
\newblock \emph{Neural Computation}, 2000.

\bibitem[Tsybakov(2008)]{tsybakov08nonparametric}
Alexandre~B. Tsybakov.
\newblock \emph{Introduction to Nonparametric Estimation}.
\newblock Springer, 2008.

\bibitem[Fox and Parker(1968)]{fox68chebyshev}
L.~Fox and I.~B. Parker.
\newblock \emph{{Chebyshev polynomials in numerical analysis}}.
\newblock Oxford U.P. cop., 1968.

\bibitem[Trefethen(2012)]{trefethen12approximation}
Lloyd~N. Trefethen.
\newblock \emph{{Approximation Theory and Approximation Practice}}.
\newblock Society for Industrial and Applied Mathematics, 2012.

\bibitem[Birg\'{e} and Massart(1995)]{birge95estimation}
Lucien Birg\'{e} and Pascal Massart.
\newblock {Estimation of integral functionals of a density}.
\newblock \emph{Ann. of Stat.}, 1995.

\bibitem[Robins et~al.(2009)Robins, Li, Tchetgen, and van~der
  Vaart]{robins09quadraticvm}
James Robins, Lingling Li, Eric Tchetgen, and Aad~W van~der Vaart.
\newblock {Quadratic semiparametric Von Mises Calculus}.
\newblock \emph{Metrika}, 69\penalty0 (2-3):\penalty0 227--247, 2009.

\bibitem[Kandasamy et~al.(2015)Kandasamy, Krishnamurthy, P{\'{o}}czos,
  Wasserman, and Robins]{kandasamy15vonmises}
Kirthevasan Kandasamy, Akshay Krishnamurthy, Barnab{\'{a}}s P{\'{o}}czos, Larry
  Wasserman, and James Robins.
\newblock {Nonparametric Von Mises Estimators for Entropies, Divergences and
  Mutual Informations}.
\newblock In \emph{NIPS}, 2015.

\bibitem[Lee(1998)]{lee1998weyl}
Woo~Young Lee.
\newblock Weyl's theorem for operator matrices.
\newblock \emph{Integral Equations and Operator Theory}, 1998.

\bibitem[Liu et~al.(2011)Liu, Xu, Gu, Gupta, Lafferty, and
  Wasserman]{liu11forest}
Han Liu, Min Xu, Haijie Gu, Anupam Gupta, John~D. Lafferty, and Larry~A.
  Wasserman.
\newblock {Forest Density Estimation}.
\newblock \emph{Journal of Machine Learning Research}, 12:\penalty0 907--951,
  2011.

\bibitem[Gin{\'e} and Guillou(2002)]{gine02rates}
Evarist Gin{\'e} and Armelle Guillou.
\newblock {Rates of strong uniform consistency for multivariate kernel density
  estimators}.
\newblock In \emph{Annales de l'IHP Probabilit{\'e}s et statistiques}, 2002.

\bibitem[Stewart and Sun(1990)]{stewart90perturbation}
G.~W. Stewart and Ji-guang Sun.
\newblock \emph{{Matrix Perturbation Theory}}.
\newblock Academic Press, 1990.

\bibitem[Paxson and Floyd(1995)]{paxson1995wide}
Vern Paxson and Sally Floyd.
\newblock {Wide area traffic: the failure of Poisson modeling}.
\newblock \emph{IEEE/ACM Transactions on Networking}, 1995.

\bibitem[H{\"u}bner et~al.(1989)H{\"u}bner, Abraham, and
  Weiss]{hubner1989dimensions}
U~H{\"u}bner, NB~Abraham, and CO~Weiss.
\newblock {Dimensions and entropies of chaotic intensity pulsations in a
  single-mode far-infrared NH 3 laser}.
\newblock \emph{Physical Review A}, 1989.

\bibitem[sle(2016)]{sleep}
{Santa Fe Time Series Competition}.
\newblock
  {\scriptsize\url{http://www-psych.stanford.edu/$\sim$andreas/Time-Series/SantaFe.html}},
  2016.
\newblock Accessed: 2016-05-17.

\bibitem[{Hashemi, B. and Trefethen, L. N.}(2016)]{hashemi16cheb3}
{Hashemi, B. and Trefethen, L. N.}
\newblock Chebfun to three dimensions.
\newblock \emph{In preparation}, 2016.

\bibitem[Singh et~al.(2004)Singh, James, and Rudary]{singh04psr}
Satinder Singh, Michael~R. James, and Matthew~R. Rudary.
\newblock {Predictive State Representations: A New Theory for Modeling
  Dynamical Systems}.
\newblock In \emph{UAI}, 2004.

\bibitem[Anandkumar et~al.(2014)Anandkumar, Ge, Hsu, Kakade, and
  Telgarsky]{anandkumar14tensor}
Animashree Anandkumar, Rong Ge, Daniel Hsu, Sham~M. Kakade, and Matus
  Telgarsky.
\newblock {Tensor Decompositions for Learning Latent Variable Models}.
\newblock \emph{JMLR}, 2014.

\end{thebibliography}
}
\bibliographystyle{unsrtnat}

\newpage
% !TEX root = ../paper.tex

\appendix
% !TEX root = ../../paper.tex

\section{A Quart-sized Review of Continuous Linear Algebra}
\label{app:cla}

In this section we introduce continuous analogues of
matrices and their factorisations. We only provide a brief quart-sized review
for what is needed in
this exposition. Chapters 3 and 4 of~\citet{townsend14computing} contains a
reservoir-sized review.

A \emph{matrix} $F\in\RR^{m\times n}$ is an $m\times n$ array of numbers where $F(i,j)$
denotes the entry in row $i$, column $j$.
We will also look at cases where either $m$ or $n$ is infinite.
A \emph{column qmatrix} (quasi-matrix) $Q\in\RR^{[a,b]\times m}$ is a collection of
$m$ functions defined on $[a,b]$ where the row index is continuous and column index is
discrete.
Writing $Q = [q_1,\dots, q_m]$ where $q_j:[a,b]\rightarrow \RR$ is the $j$\ssth
function, $Q(y,j) = q_j(y)$ denotes the  value of the $j$\ssth function at $y\in[a,b]$.
$Q^\top\in\RR^{m\times [a,b]}$ denotes a row qmatrix with $Q^\top(j,y) = Q(y,j)$.
A \emph{cmatrix} (continous-matrix) $C\in\RRm{[a,b]}{[c,d]}$ is a two dimensional
function where both
the row and column indices are continuous and $C(y,x)$ is value of the
function at $(y,x) \in [a,b]\times[c,d]$.
$C^\top \in \RRm{[c,d]}{[a,b]}$ denotes its transpose with $C^\top(x,y) = C(y,x)$.

Qmatrices and cmatrices permit all matrix multiplications with suitably defined inner
products. Let
$F \in \RRm{m}{n}$, $Q \in \RRm{[a,b]}{m}$, $P\in \RRm{[a,b]}{n}$,
$R\in \RRm{[c,d]}{m}$ and $C\in\RRm{[a,b]}{[c,d]}$. It follows that
$F(:,j)\in \RRv{m}$, $Q(y,:)\in\RRm{1}{m}$, $Q(:,i)\in\RRv{[a,b]}$,
$C(y,:)\in\RRm{1}{[c,d]}$ etc.
Then the following hold:
\vspace{-0.02in}
\begin{itemize}
\item
\vspace{-0.05in}
$QF = S \in \RRm{[a,b]}{n}\quad$ where
$\quad S(y,j) = Q(y,:) F(:,j) = \sum_{k=1}^m  Q(y,k)F(i,k)$.
\item
\vspace{-0.05in}
$Q^\top P = H\in \RRm{m}{n}\quad$ where
$\quad H(i,j) = Q(:,j)^\top P(:,j) =  \int_a^b Q^\top(i,s)P(s,j)\ud s$.
\item
\vspace{-0.05in}
$QR^\top = D\in \RRm{[a,b]}{[c,d]}\quad$ where
$\quad D(y,x) = Q(y,:)R(x,:)^\top =\sum_{1}^m Q(y,k)R^\top(k,x)$.
% \vspace{-0.05in}
\item
\vspace{-0.05in}
$CR = T\in\RRm{[a,b]}{m}\quad$ where
$\quad T(y,j) = C(y,:)R(:,j) = \int_c^d C(y,s)R(s,j)\ud s$.
\vspace{-0.05in}
\end{itemize}
% TODO: For final report
% Figure ... illustrates the second matrix multiplication.
% \insertQmCmFigure
\vspace{-0.02in}
Here, the integrals are with respect to the Lebesgue measure.

A cmatrix has a singular value decomposition (SVD).
If $C\in\RRm{[a,b]}{[c,d]}$, an SVD of $C$ is the sum
$
C(y,x) %= U\Sigma V^\top
= \sum_{j=1}^\infty \sigma_j u_j(y)v_j(x),
$
which converges in $\Ltwo$.
% \[
% C(y,x) %= U\Sigma V^\top
% = \sum_{j=1}^\infty \sigma_j u_j(y)v_j(x),
% \]
Here $\sigma_1\geq \sigma_2\geq \dots$.
are the singular values of $C$.
% $\{u_j\}_{j\geq 1}$ are the left singular vectors and $\{v_j\}_{j\geq 1}$ are the right
% singular vectors.  $\{u_j\}_{j\geq 1}, \{v_j\}_{j\geq 1}$ are orthonormal
% % bases for functions on $[a,b]$ and $[c,d]$ respectively, i.e.
% bases $\Ltwo([a,b])$ and $\Ltwo([c,d])$ respectively, i.e.
% $\int_a^b u_j(s)u_k(s)\ud s = \indfone(j=k)$.
$\{u_j\}_{j\geq 1}$ and $\{v_j\}_{j\geq 1}$ are the left and right
singular vectors and form  orthonormal
bases for $\Ltwo([a,b])$ and $\Ltwo([c,d])$ respectively, i.e.
$\int_a^b u_j(s)u_k(s)\ud s = \indfone(j=k)$.
It is known that the SVD of a cmatrix exists uniquely with $\sigma_j \rightarrow 0$,
 and continuous singular vectors
(Theorem 3.2,~\citep{townsend14computing}).
Further, if $C$ is Lipshcitz continuous w.r.t both
variables then
the SVD is absolutely and uniformly convergent.
Writing the singular vectors as infinite qmatrices
$U = [u_1, u_2\dots ], V=[v_1,v_2\dots]$,
and $\Sigma = \diag(\sigma_1,\sigma_2\dots )$ we can write the SVD as,
\[
C = U\Sigma V^\top = \sum_{j=1}^\infty \sigma_j U(:,j)V(:,j)^\top.
\]
If only $m<\infty$ singular values are nonzero then we say that $C$ is of rank $m$.
% TODO: Figure ... illustrates a cmatrix SVD.
The SVD of a Qmatrix $Q\in\RRm{[a,b]}{m}$ is,
$
Q = U\Sigma V^\top = \sum_{j=1}^m \sigma_j U(:,j)V(:,j)^\top,
$
where $U\in \RRm{[a,b]}{m}$ and $V\in\RRm{m}{m}$ have orthonormal columns
and $\Sigma = \diag(\sigma_1,\dots,\sigma_m)$ with
$\sigma_1\geq \sigma_2\geq \dots \geq \sigma_m \geq 0$.
The SVD of a qmatrix also exists uniquely (Theorem 4.1,~\citep{townsend14computing}).
The rank of a column qmatrix is the number of linearly independent columns
(i.e. functions) and is equal to the number of nonzero singular values.

Finally, the pseudo inverse of the cmatrix $C$ is $C^\dagger = V \Sigma^{-1}U^\top$ with
$\Sigma^{-1} = \diag( 1/\sigma_1, 1/\sigma_2, \dots )$.
The $p$-operator norm of a cmatrix, for $1\leq p \leq \infty$ is
$\|C\|_p = \sup_{\|x\|_p = 1} \|Cx\|_p$ where
$x\in \RRv{[c,d]}$, $Cx \in \RRv{[a,b]}$, $\|x\|^p_p = \int_c^d (x(s))^p\ud s$ for
$p< \infty$ and $\|x\|_\infty = \sup_{s\in[c,d]} x(s)$.
The Frobenius norm of a cmatrix is $\|C\|_F = \left( \int_a^b\int_c^d C(y,x)^2 \ud x
\ud y \right)^{1/2}$.
It can be shown that $\|C\|_2 = \sigma_1$ and $\|C\|^2_F = \sum_j \sigma_j^2$ where
$\sigma_1\geq \sigma_2 \geq \dots$ are its singular values.
Note that analogous relationships hold with finite matrices.
The pseudo inverse and norms of a qmatrix are similarly defined and similar
relationships hold with its singular values.

% In what follows we will use $\one_{[a,b]}, \zero_{[a,b]}$ to denote functions on
% $[a,b]$ which take values $1$ and $0$ everywhere and $\one_m,\zero_m$ to denote
% $m$-vectors of $1$'s and $0$'s.
\textbf{Notation:}
In what follows we will use $\one_{[a,b]}$ to denote the function
 taking value $1$ everywhere in $[a,b]$ and $\one_m$ to denote
$m$-vectors of $1$'s.
% When evident from context, we will drop the subscripts.
When we are dealing with $L^p$ norms of a function we will explicitly use the
subscript $L^p$ to avoid confusion with the operator/Frobenius norms of qmatrices and
cmatrices.
For example, for a cmatrix $\|C\|^2_{\Ltwo} = \int\int C(\cdot,\cdot)^2 = \|C\|^2_F$.
As we have already done,
throughout the paper we will overload notation for inner products, multiplications
and pseudo-inverses depending on whether they hold for matrices, qmatrices or
cmatrices.
E.g. when $p,q\in \RR^m, p^\top q  = \sum_1^m p_i q_i$ and when $p,q\in \RR^{[a,b]}$,
$p^\top q = \int_a^b p(s)q(s)\ud s$.
$\PP$ will be used to denote probabilities of events while $p$ will denote
probability density functions (pdf).

% !TEX root = ../../paper.tex

\section{Some Perturbation Theory Results for Continuous Linear Algebra}
\label{app:PT}

We recommend that readers unfamiliar with continuous linear algebra first read
the review in Appendix~\ref{app:cla}.
Throughout this section $\Lcal(\cdot)$ maps a matrix (including q/cmatrices) to its eigenvalues.
Similarly, $\sigma(\cdot)$ maps a matrix to its singular values.
When we are dealing with infinite sequences and qmatrices ``=" refers to convergence
in $\Ltwo$. When dealing with infinite sequences and cmatrices,
``=" refers to convergence in the operator norm.
For all theorems, we follow the template of~\citet{stewart90perturbation} for
the matrix case and hence try to stick with their notation.

Before we proceed, we introduce the ``cmatrix" $\Izo$ on $[0,1]$.
For any $u\in\RRv{[0,1]}$ this is the operator which satisfies $\Izo u = u$.
That is, $(\Izo u)(y) = \int_0^1 \Izo(y,x)u(x)\ud x = u(y)$.
Intuitively, it can be thought of as the Dirac delta function along the
diagonal, $\delta(x-y)$.
Let $Q = [q_1, q_2, \dots, ] \in\RRm{[0,1]}{\infty}$ be a qmatrix containing an
orthonormal basis for $[0,1]$ and $Q_k \in\RRm{[0,1]}{k}$ denote the first $k$ columns
of $Q$. We make note of the following observation.

\insertprespacing
\begin{theorem}
\label{thm:QQttheorem}
$Q_kQ_k^\top \rightarrow \Izo$ as $k\rightarrow\infty$. Here convergence is in the
operator norm.
\begin{proof}
We need to show that for all $x\in\RRv{[0,1]}$, $\|Q_kQ_k^\top x - x\|_2 \rightarrow
0$. Let $x = Q\alphab = \sum_{k=1}^\infty \alpha_k q_k$ be the representation of $x$ in the
$Q$-basis. Here $\alphab = (\alpha_1,\alpha_2,\dots)$ satisfies $\sum_k\alpha_k^2 < \infty$.
We then have $\|Q_kQ_k^\top x - x\|_2^2 = \sum_{j=k+1}^\infty \alpha_{j}^2
\rightarrow 0$ by the properties of sequences in $\ell^2$.
\end{proof}
\end{theorem}
\insertpostspacing

We now proceed to our main theorems. We begin with a series of intermediary results.

\insertprespacing
\begin{theorem}
\label{thm:Tsingular}
Let $X\in\RRm{[0,1]}{m}$. Define the \textbf{linear} operator
$\Toper(X) = AX - XB$ where $A \in \RRm{[0,1]}{[0,1]}$ and $B\in\RRm{m}{m}$
are a \textbf{square} cmatrix and matrix, respectively.
Then, $T$ is nonsingular if and only if $\Lcal(A)\cap\Lcal(B) = \emptyset$.
\end{theorem}
\insertpostspacing
\begin{proof}
Assume $\lambda\in\Lcal(A)\cup\Lcal(B)$.
Then, let $Ap = \lambda p$, $q^\top B = \lambda q^\top$ where $p\in\RRv{[0,1]}$
and $q\in\RR^m$. Then $\Toper(pq^\top) = \zero$ and $\Toper$ is singular.
This proves one side of the theorem.

Now, assume that $\Lcal(A)\cap\Lcal(B) = \emptyset$. As the operator is linear,
it is sufficient to show that $AX-XB=C$ has a unique solution for any
$C \in \RRm{[0,1]}{m}$.
Let the Schur decomposition of $B$ be $Q=V^\top B V$ where $V$ is orthogonal
and $Q$ is upper triangular. Writing $Y=XV$ and $D=CV$ it is sufficient to show
that $AY-YQ=D$ has a unique solution. We write
\begin{equation*}
Y = (y_1, y_2, \dots y_m)\in\RRm{[0,1]}{m} \text{ and }
D = (d_1,d_2,\dots, d_m)\in\RRm{[0,1]}{m}
\end{equation*}
and use an inductive argument over the columns of $Y$.

The first column of $Y$ is given by
$Ay_1 - Q_{11}y_1 = (A - Q_{11}\Izo)y_1 = d_1$. Since $Q_{11}\in\Lcal(B)$ and
$\Lcal(A)\cap\Lcal(B)$ is empty $(A - Q_{11}\Izo)$ is nonsingular. Therefore
$y_1$ is uniquely determined by inverting the cmatrix
(see Appendix~\ref{app:cla}).
Assume $y_1,y_2\dots,y_{k-1}$ are uniquely determined.
Then, the $k$\ssth column is given by
$(A - Q_{kk}\Izo)y_k = d_k + \sum_{i=1}^{k-1}Q_{ik}y_i$.
Again, $(A - Q_{kk}\Izo)$ is nonsingular by assumption, and hence
this uniquely determines $y_k$.
\end{proof}

\insertprespacing
\begin{corollary}
\label{cor:Teigs}
Let $\Toper$ be as defined in Theorem~\ref{thm:Tsingular}. Then
\begin{equation*}
\Lcal(\Toper) = \Lcal(A) - \Lcal(B) = \{\alpha -\beta: \alpha\in\Lcal(A),
\beta\in\Lcal(B)\}.
\end{equation*}
\end{corollary}
\begin{proof}
If $\lambda\in\Lcal(\Toper)$ there exists $X$ such that
$(A-\lambda\Izo)X - XB = \zero$. Therefore, by Theorem~\ref{thm:Tsingular}
there exists $\alpha\in\Lcal(A)$ and $\beta\in\Lcal(B)$ such that
$\lambda = \alpha-\beta$. Therefore, $\Lcal(\Toper) \subset \Lcal(A)-\Lcal(B)$.

Conversely, consider any  $\alpha\in\Lcal(A)$ and $\beta\in\Lcal(B)$.
Then there exists $a\in\RRv{[0,1]}$, $b\in\RRv{m}$ such that $Aa = \alpha a$ and
$b^\top B = \beta b^\top$. Writing $X= ab^\top$ we have
$AX - XB = (\alpha-\beta)ab^\top$. Therefore, $\Lcal(A)-\Lcal(B)\subset
\Lcal(\Toper)$.
\end{proof}

\insertprespacing
\begin{theorem}
\label{thm:frobSep}
Let $\Toper$ be as defined in Theorem~\ref{thm:Tsingular}. Then
\begin{equation}
\inf_{\|X\|_\frob =1} \|\Toper(X)\|_\frob = \min\Lcal(\Toper) = \min|\Lcal(A) - \Lcal(B)|.
\end{equation}
\end{theorem}
\begin{proof}
For any qmatrix $P = (p_1, p_2,\dots,p_m)\in\RRm{[0,1]}{m}$
let $\vect(P) = [p_1^\top, p_2^\top, \dots, p_m^\top]^\top \in \RRm{[0,m]}{1}$
be the concatenation of all functions.
Then $\vect(XB) = \vec{B}\vect(X)$ where,
\begin{equation*}
\vec{B} =
\begin{bmatrix}
B_{11}\Izo & B_{21}\Izo & \cdots & B_{m1}\Izo \\
B_{12}\Izo & B_{22}\Izo & \cdots & B_{m2}\Izo \\
\vdots & \vdots & \ddots & \vdots \\
B_{1m}\Izo & B_{2m}\Izo & \cdots & B_{mm}\Izo
\end{bmatrix}
\in \RRm{[0,m]}{[0,m]}.
\end{equation*}
Here $\Izo$ have been translated and should be interpreted as being a dirac delta
function on that block.
Similarly, $\vect(AX) = \vec{A}\vect(X)$ where $\vec{A} = \diag(A, A, \dots,
A)\in\RRm{[0,m]}{[0,m]}$.
Therefore $\vect(\Toper(X)) =(\vec{A}-\vec{B})\vec{X}$.
Now noting that $\|X\|_\frob = \|\vect(X)\|_2$ we have,
\begin{equation*}
\inf_{\|X\|_\frob =1} \|\Toper(X)\|_\frob =
\inf_{\|\vect(X)\|_2 =1} \|\vect(\Toper(X))\|_2 =
\min |\Lcal(\vec{A}-\vec{B})|.
\end{equation*}
The theorem follows by noting that the eigenvalues of $(\vec{A}-\vec{B})$ are
the same as those of $\Lcal(\Toper)$.
\end{proof}

\insertprespacing
\begin{theorem}
\label{thm:csone}
Let $\Xpone,\Ypone\in\RRm{[0,1]}{\ell}$ have orthonormal columns. Then,
there exist $Q\in\RRm{\infty}{[0,1]}$ and $\Upoo,\Vpoo\in\RRm{\ell}{\ell}$ such
that the following holds,
\begin{equation*}
Q\Xpone\Upoo = \begin{bmatrix}
I_\ell \\
\zero
\end{bmatrix}
\;\in\RRm{\infty}{\ell}
,\hspace{0.4in}
Q\Ypone\Vpoo = \begin{bmatrix}
\Gamma \\ \Sigma \\ \zero
\end{bmatrix}
\;\in\RRm{\infty}{\ell}.
\end{equation*}
Here $\Gamma = \diag(\gamma_1,\dots,\gamma_\ell)$,
$\Sigma = \diag(\sigma_1,\dots,\sigma_\ell)$ and they satisfy
\begin{equation*}
0 \leq \gamma_1\leq \dots \leq \gamma_\ell,\,
\sigma_1 \geq \dots \geq \sigma_\ell\geq 0,\, \text{and}\
\gamma_i^2 + \sigma_i^2 = 1,\, i = 1, \dots, \ell.
\end{equation*}
\end{theorem}
\begin{proof}
Let $\Xptwo,\Yptwo\in \RRm{[0,1]}{\infty}$ be orthonormal bases for the
complementary subspaces of $\Rcal(\Xpone), \Rcal(\Ypone)$, respectively.
Denote $X = [\Xpone, \Xptwo]$, $Y = [\Ypone, \Yptwo]$ and
\begin{equation*}
W = X^\top Y = \begin{pmatrix}
W_{11} & W_{12} \\ W_{21} & W_{22}
\end{pmatrix}
\in \RRm{\infty}{\infty},
\end{equation*}
where $W_{11} = \Xpone^\top \Ypone \in\RRm{\ell}{\ell}$ and the rest are
defined accordingly.
Now, using Theorem~5.1 from~\citep{stewart90perturbation} there exist
orthogonal matrices $U = \diag(\Upoo, \Uptt), V=\diag(\Vpoo,\Vptt)$ where
$\Upoo,\Vpoo\in\RRm{\ell}{\ell}$ and
$\Uptt,\Vptt\in\RRm{\infty}{\infty}$ such that the following holds,
\begin{equation*}
U^\top W V =
\begin{pmatrix}
\Gamma & -\Sigma & \zero \\
\Sigma & \Gamma & \zero \\
\zero  &  \zero  &  I_\infty
\end{pmatrix}
\in\RRm{\infty}{\infty}.
\end{equation*}
Here $\Gamma, \Sigma$ satisfy the conditions of the theorem.
Now set
$\;\widehat{X} = [\widehat{X}_1, \widehat{X}_2]$,
$\;\widehat{Y} = [\widehat{Y}_1, \widehat{Y}_2]$
where
$\;\widehat{X}_1 = \Xpone\Upoo$,
$\;\widehat{X}_2 = \Xptwo\Upoo$,
$\;\widehat{Y}_1 = \Ypone\Vpoo$,
$\;\widehat{Y}_2 = \Yptwo\Vpoo$.
Then, $\widehat{X}^\top Y = U^\top W V$. Setting $Q = \widehat{X}^\top$ and setting
$\Upoo, \Vpoo$ as above yields,
\begin{equation*}
Q\Xpone\Upoo =
\begin{pmatrix}
\Upoo^\top \Xpone^\top \\ \Uptt^\top \Xptwo^\top
\end{pmatrix}
\Xpone \Upoo =
\begin{bmatrix} I_\ell \\ \zero \end{bmatrix}, \quad
Q\Ypone\Vpoo =
\begin{pmatrix}
\Upoo^\top \Xpone^\top \\ \Uptt^\top \Xptwo^\top
\end{pmatrix}
\Ypone\Vpoo =
% \begin{pmatrix}
% \Upoo^\top \Xpone^\top \Ypone \Upoo \\
% \Uptt^\top \Xptwo^\top \Ypone \Upoo \\
% \end{pmatrix} =
\begin{bmatrix} \Gamma \\ \Sigma \\ \zero \end{bmatrix}
\end{equation*}
where
$\Upoo^\top \Xpone^\top \Ypone \Upoo = \Gamma$,
$\Uptt^\top \Xptwo^\top \Ypone \Upoo = [\Sigma^\top, \zero^\top]^\top$ from
the decomposition of $U^\top W V$.
\end{proof}

\begin{remark}
\citet{stewart90perturbation} prove Theorem~5.1 for a finite unitary $W$.
However, it is straightforward to verify that the same holds if $W$ is a
unitary operator on the $\ell^2$ sequence space, i.e., Theorem~5.1 is valid
for (countably) infinite matrices.
\end{remark}

\insertprespacing
\begin{definition}[Canonical Angles]
\label{defn:canonicalAngle}
Let $\Xcal,\Ycal$ be $\ell$ dimensional subspaces of the same dimension for
functions on $[0,1]$ and $\Xpone,\Ypone\in\RRm{[0,1]}{\ell}$ be orthonormal
functions spanning these subspaces. Then the canonical angles between
$\Xcal$ and $\Ycal$ are the diagonals of the matrix
$\canangle{\Xcal}{\Ycal} \defeq \sin^{-1}(\Sigma)$ where $\Sigma$ is from
Theorem~\ref{thm:csone}.
It follows that $\cos\canangle{\Xcal}{\Ycal} =\Gamma$ where $\sin$ and $\cos$
are in the usual trigonometric sense and satisfy $\cos^2(x) + \sin^2(x) = 1$.
\end{definition}

\insertprespacing
\begin{corollary}
\label{cor:csCompute}
Let $\Xcal,\Ycal,\Xpone, \Ypone$ be as in Definition~\ref{defn:canonicalAngle}
and $\Xptwo,\Yptwo$ be orthonormal functions for their complementary spaces.
Then, the nonzero singular values of $\Xptwo^\top\Ypone$ are the sines of the
nonzero canonical angles between $\Xcal, \Ycal$.
The singular values of $\Xpone^\top\Ypone$ are the cosines of the nonzero
canonical angles.
\end{corollary}
\begin{proof}
From the proof of Theorem~\ref{thm:csone},
\begin{equation*}
\Xptwo^\top \Ypone = \Uptt \begin{pmatrix} \Sigma \\ \zero\end{pmatrix}
\Upoo^\top, \qquad
\Xpone^\top \Ypone = \Upoo\Gamma\Upoo^\top.
\end{equation*}
Since $\Upoo,\Uptt$ are orthogonal, the above are the SVDs of
$\Xptwo^\top\Ypone$ and $\Xpone^\top\Ypone$.
\end{proof}

\insertprespacing
\begin{theorem}
\label{thm:cstwo}
Let $\Xcal,\Ycal$ be $\ell$ dimensional subspaces of functions on $[0,1]$ and
$\Xpone,\Ypone\in\RRm{[0,1]}{l}$ be an orthonormal bases.
Let $\sin\canangle{\Xcal}{\Ycal} = \diag(\sigma_1,\dots,\sigma_\ell)$.
Denote $P_\Xcal=\Xpone\Xpone^\top$ and $P_\Ycal = \Ypone\Ypone^\top$.
Then, the singular values of $P_\Xcal(\Izo - P_\Ycal)$ are
$\sigma_1,\sigma_2,\dots,\sigma_\ell,0,0,\dots$.
\end{theorem}
\begin{proof}
By Theorem~\ref{thm:csone}, there exists
$Q\in\RRm{\infty}{[0,1]}$, $\Upoo,\Vpoo\in\RRm{\ell}{\ell}$,
such that
\begin{align*}
&QP_\Xcal(\Izo-P_\Ycal)Q^\top =
Q\Xpone\Xpone^\top Q^\top Q (\Izo-\Ypone\Ypone^\top)Q^\top
\\
&\hspace{0.3in}=
(Q\Xpone\Upone)(\Upone^\top\Xpone^\top Q^\top)
 (\Izo-Q\Ypone\Vpoo(\Vpoo^\top\Ypone^\top Q^\top) )
= \begin{bmatrix}\Sigma \\\zero\\ \zero \end{bmatrix}
\begin{bmatrix}\Sigma & -\Gamma & \zero \end{bmatrix}
\end{align*}
Here we have used $\Izo = Q^\top Q$.
The proof of this uses a technical argument involving the dual space of the
class of operators described by cmatrices.
(In the discrete matrix case this is similar to how the outer product of a
complete orthonormal basis results in the identity $UU^\top = I$.)
The last step follows from Theorem~\ref{thm:csone} and some algebra.
Noting that $\begin{bmatrix}\Sigma & -\Gamma & \zero \end{bmatrix}$ has
orthonormal rows, it follows that the singular values of
$P_\Xcal(\Izo-P_\Ycal)$ are $\Sigma$.
\end{proof}

\insertprespacing
\begin{theorem}
\label{thm:frobGap}
Let $A\in\RRm{[0,1]}{[0,1]}$ satisfy,
\begin{equation*}
A =
\begin{bmatrix} \Xpone & \Xptwo  \end{bmatrix}
\begin{bmatrix} L_1 & \zero \\ \zero & L_2  \end{bmatrix}
\begin{bmatrix} \Xpone^\top \\ \Xptwo^\top  \end{bmatrix}
\end{equation*}
where $\Xpone \in \RRm{[0,1]}{\ell}$ and $[\Xpone,\, \Xptwo]$ is unitary.
Let $Z\in\RRm{[0,1]}{m}$ and $T = AZ-ZB$ where $B\in\RRm{m}{m}$.
Let $\delta = \min|\Lcal(L_2) - \Lcal(B)| > 0$. Then,
\begin{equation*}
\big\| \sin \canangle{\Rcal(\Xpone)}{\Rcal(Z)}\|_\frob
\;\leq\; \frac{ \|T\|_\frob}{\delta}.
\end{equation*}
\end{theorem}
\begin{proof}
First note that $\Xptwo^\top T = L_2\Xptwo^\top Z - \Xptwo^\top Z B$. The claim
follows from Theorems~\ref{thm:frobSep} and~\ref{cor:csCompute}.
\begin{equation*}
\big\| \sin \canangle{\Rcal(\Xpone)}{\Rcal(Z)}\|_\frob
\,=\, \|\Xptwo^\top Z\|_\frob
\,\leq \frac{\|\Xptwo^\top T\|_\frob}{\min |\Lcal(L_2) - \Lcal(B)|}
\leq \frac{\|T\|_\frob}{\delta}.
\end{equation*}
\end{proof}

\insertprespacing
\begin{theorem}[\textbf{Wedin's Sine Theorem for cmatrices -- Frobenius form}]
\label{thm:wedin}
Let $A,\Atilde,E \in\RRm{[0,1]}{[0,1]}$ with $\Atilde = A + E$. Let $A,\Atilde$ have
the following conformal partitions,
\[
A \,=\, \begin{bmatrix} \Upone & \Uptwo \end{bmatrix}
\begin{bmatrix} \Sigma_1 & \zero \\ \zero & \Sigma_2 \end{bmatrix}
\begin{bmatrix} \Vpone^\top \\ \Vptwo^\top \end{bmatrix}\,,
\hspace{0.4in}
\Atilde \,=\, \begin{bmatrix} \Utone & \Uttwo \end{bmatrix}
\begin{bmatrix} \Sigmatilde_1 & \zero \\ \zero & \Sigmatilde_2 \end{bmatrix}
\begin{bmatrix} \Vtone^\top \\ \Vttwo^\top \end{bmatrix}\,.
\]
where $\Upone,\Utone\in\RRm{[0,1]}{m}$,  $\,\Vpone,\Vtone\in\RRm{[0,1]}{m}$
and  $\Uptwo,\Uttwo\in\RRm{[0,1]}{\infty}$,  $\,\Vptwo,\Vttwo\in\RRm{[0,1]}{\infty}$.
Let $R = A\Vtone - \Utone\Sigmatilde_1\in\RRm{[0,1]}{m}$ and
$S = A^\top\Utone - \Vtone\Sigmatilde_1\in\RRm{[0,1]}{m}$.
% Let $\delta = \min|\sigma(\Sigmatilde_1) - \sigma(\Sigma_2)| > 0$.
% Then,
Assume there exists $\delta>0$ such that,
$\min|\sigma(\Sigmatilde_1) - \sigma(\Sigma_2)| \geq \delta$
and
$\min|\sigma(\Sigmatilde_1)| \geq \delta$.
Let $\Phi_1,\Phi_2$ denote the canonical angles between
$(\Rcal(\Upone),\Rcal(\Utone))$ and
$(\Rcal(\Vpone),\Rcal(\Vtone))$ respectively.
Then,
\begin{equation*}
\sqrt{\|\sin\Phi_1\|^2_\frob + \|\sin\Phi_2\|^2_\frob }
\;\leq\;
\frac{ \sqrt{\vphantom{\big(} \|R\|_\frob^2 + \|S\|^2_\frob }}{\delta}.
\end{equation*}
\end{theorem}
\begin{remark}
The two conditions on $\delta$ are needed because the theorem doesn't require
$\Sigma_1,\Sigma_2,\Sigmatilde_1,\Sigmatilde_2$ to be ordered. If they were
ordered, then it reduces to
$\delta = \min|\sigma(\Sigmatilde_1) - \sigma(\Sigma_2)| > 0$.
\end{remark}
\begin{proof}
First define $Q\in\RRm{[0,2]}{[0,2]}$,
\begin{equation*}
Q = \begin{bmatrix} \zero & A \\ A^\top & \zero \end{bmatrix}.
\end{equation*}
It can be verified that if $u_i\in\RRv{[0,1]}, v_i\in\RRv{[0,1]}$ are a
left/right singular vector pair with singular value $\sigma_i$, then
$(u_i,v_i)\in\RRv{[0,2]}$ is an eigenvector with eigenvalue $\sigma_i$ and
$(u_i,-v_i)\in\RRv{[0,2]}$ is an eigenvector with eigenvalue $-\sigma_i$.
Writing,
\begin{equation*}
X = \frac{1}{\sqrt{2}} \begin{pmatrix}
\Upone & \Upone  \\ \Vpone & -\Vpone \end{pmatrix},
\hspace{0.3in}
Y = \frac{1}{\sqrt{2}} \begin{pmatrix}
\Uptwo & \Uptwo  \\ \Vptwo & -\Vptwo \end{pmatrix},
\end{equation*}
we have,
\begin{equation*}
Q =
\begin{bmatrix} X & Y \end{bmatrix}
\begin{bmatrix}
\Sigma_1 & \zero & \zero &\zero  \\
\zero & -\Sigma_1 & \zero &\zero  \\
\zero & \zero & \Sigma_2 &\zero  \\
\zero & \zero & \zero & -\Sigma_2
\end{bmatrix}
\begin{bmatrix} X^\top \\ Y^\top \end{bmatrix}.
\end{equation*}
We similarly define $\Qtilde, \Xtilde,\Ytilde$ for $\Atilde$.
Now let $T = Q\Xtilde - \Xtilde\diag(\Sigmatilde_1, -\Sigmatilde_1)$.
We will apply Theorem~\ref{thm:frobGap} with
$L_1 = \diag(\Sigma_1,-\Sigma_1)$, $L_2=\diag(\Sigma_2,-\Sigma_2)$,
$Z = \tilde{X}$, $B = \diag(\Sigmatilde_1, -\Sigmatilde_1)$.
Then, using the conditions on $\delta$ gives us,
\begin{equation*}
\big\| \sin \canangle{\Rcal(X)}{\Rcal(\Xtilde)}\|_\frob
\;\leq\; \frac{ \|T\|_\frob}{\delta}.
\end{equation*}
It is straightforward to verify that
$\|T\|_\frob^2 = \|R\|^2_\frob + \|S\|^2_\frob$.
To conclude the proof, first note that
\begin{equation*}
% \sin \canangle{\Rcal(X)}{\Rcal(\Xtilde)} =
XX^\top(\Izt - YY^\top) =
% \begin{bmatrix} \Upone\Upone^\top & \zero \\
% \zero & \Izo - \Vpone\Vpone^\top \end{bmatrix}
% \begin{bmatrix} \Izo - \Utone\Utone^\top & \zero \\
% \zero & \Vtone\Vtone^\top \end{bmatrix}
\begin{bmatrix} (\Upone\Upone^\top)(\Izo - \Utone\Utone^\top) & \zero \\
\zero & (\Vpone\Vpone^\top)(\Izo - \Vtone\Vtone^\top) \end{bmatrix}
\end{equation*}
Now, using Theorem~\ref{thm:cstwo} we have
$\|\sin \canangle{\Rcal(X)}{\Rcal(\Xtilde)}\|^2_\frob
= \|\sin \Phi_1^2\|^2_\frob  + \|\sin \Phi_2^2\|^2_\frob$.
\end{proof}

We can now prove Lemma~\ref{lem:wedin} which follows directly from
Theorem~\ref{thm:wedin}.
\begin{proof}[\textbf{Proof of Lemma~\ref{lem:wedin}}]
Let $\Utilde_\perp\in\RRm{[0,1]}{m}$ be an orthonormal basis for the complementary
subspace of $\Rcal(\Utilde)$. Then, by Corollary~\ref{cor:csCompute},
$\|\Utilde_\perp^\top U\|_\frob^2 =
\|\sin\canangle{\Rcal(\Utilde)}{\Rcal(U)}\|^2_\frob$,
$\|\Vtilde_\perp^\top V\|_\frob^2 =
\|\sin\canangle{\Rcal(\Vtilde)}{\Rcal(V)}\|^2_\frob$.
For $R,S$ as defined in Theorem~\ref{thm:wedin}, we have.
$\|R\|^2_\frob,\|S\|^2_\frob <\|E\|^2_\frob$.
% $\|\Utilde_\perp^\top U\|_\frob^2 \leq 2\|E\|^2_\frob/\delta$.
The lemma follows via the
$\sin$--$\cos$ relationships for canonical angles,
\begin{equation*}
\min \sigma(\Utilde^\top U)^2
= 1- \max \sigma(\Utilde_\perp^\top U)^2
\geq
 1- \|\Utilde_\perp^\top U\|_\frob^2
\geq 1 - \frac{2\|E\|^2_\frob}{\delta^2}.
\end{equation*}
where $\delta = \sigma_m(A)$.
% \vspace{-0.5in}
\end{proof}

Next we prove the pseudo-inverse theorem. Recall that for $A\in\RRm{[0,1]}{m}$ the
SVD is $A = U\Sigma V^\top$ where $U\in\RRm{[0,1]}{m}$, $\Sigma\in\RRm{m}{m}$ and
$V\in\RRm{m}{m}$ where $U, V$ have orthonormal columns. Denote its pseudo-inverse
by $A^\dagger = V\Sigma^{-1}U^\top$.

\begin{proof}[\textbf{Proof of Lemma~\ref{lem:pinv}}]
% The proof of the pseudoinverse theorem also replicates the finite matrix
% analysis in~\citet{stewart90perturbation} by replacing standard equalities
% with the appropriate notion of convergence when appropriate.
Let $A=U\Sigma V$ be the SVD of $A$ and $\Atilde = \Utilde\Sigmatilde\Vtilde$ be the
SVD of $\Atilde$. Let $\Ptilde = \Utilde\Utilde^\top$, $R = VV^\top$,
$\Rtilde = \Vtilde\Vtilde^\top$, $P_\perp = \Izo - UU^\top$,
$\Rtilde_\perp = \Izo - \Vtilde\Vtilde^\top$ and $P=UU^\top$. We then have,
\begin{align*}
\Atilde^\dagger - A^\dagger \;&=\;\;
- \Atilde^\dagger \Ptilde E R A^\dagger \,+\,
(\Atilde^\top\Atilde)^\dagger \Rtilde E^\top P_\perp \,+\,
\Rtilde_\perp E P (AA^\top)^\dagger \\
\|\Atilde^\dagger - A^\dagger\|_2 \;&\leq\;\;
  \|\Atilde^\dagger\|_2\|E\|_2\|A^\dagger\|_2 \,+\,
\|(\Atilde^\top\Atilde)^\dagger\|_2\|E\|_2
  \,+\, \|E\|_2 \|(AA^\top)^\dagger\|_2 \\
&= \;\; \left( \|\Atilde^\dagger\|_2\|A^\dagger\|_2 + \|\Atilde^\dagger\|_2^2
  + \|A^\dagger\|_2^2 \right) \|E\|_2
\;\leq\; 3 \max\{ \|\Atilde\|_2^2, \|A\|_2^2 \} \|E\|_2
\end{align*}
The first step is obtained by substitutine for $\Ptilde, E, R, \Rtilde, P_\perp,
\Rtilde_\perp$ and $P$, the second step uses the triangle inequality, and the third
step uses $\Atilde^\top \Atilde = U\Sigma^2 U^\top$, $AA^\top = V\Sigma^2 V^\top$.
\end{proof}

\begin{remark}
$P, \Ptilde, R, \Rtilde$ can be shown to be the projection operators to
$\Rcal(A)$, $\Rcal(\Atilde)$, $\Rcal(A^\top)$ and $\Rcal(\Atilde^\top)$. Here,
$\Rcal(A) = \{Ax; x\in \RRv{m}\}\subset \RRv{[0,1]}$ is the range of $A$.
$\Rcal(\Atilde)\subset\RRv{[0,1]}$, $\Rcal(A^\top)\subset\RRv{m}$ and
$\Rcal(\Atilde^\top)\subset\RRv{m}$ are defined similarly.
$P_\perp$, $\Rtilde_\perp$ are the complementary projectors of $P,\Rtilde$.
\end{remark}

Finally, we state an analogue of Weyl's theorem for cmatrices which bounds the
difference in the singular values in terms of the operator norm of the
perturbation. While Weyl's theorem has been studied for general
operators~\citep{lee1998weyl}, we use the form below
from~\citet{townsend14computing} for cmatrices.

\insertprespacing
\begin{lemma}[Weyl's Theorem for Cmatrices,~\citep{townsend14computing}.]
Let $A, E\in \RRm{[a,b]}{[c,d]}$ and $\tilde{A} = A+E$. Let
the singular values of $A$ be $\sigma_1\geq \sigma_2,\dots$ and those of
$\tilde{A}$ be $\tilde{\sigma}_1\geq\tilde{\sigma}_2,\dots$.
Then,
\[
|\sigma_i - \tilde{\sigma}_i| \leq \|E\|_2 \quad \forall i\geq 1.
\]
\label{lem:weyl}
\end{lemma}

% !TEX root = ../../paper.tex

\section{Concentration of Kernel Density Estimation}
\label{sec:concKDE}

We will first define the H\"older
class in high dimensions.

\insertprespacing
\begin{definition}
\label{lem:dHolder}
Let $\Xcal\subset\RR^d$ be a compact space. For any $r=(r_1,\dots,r_d)$,
$r_i\in\NN$, let $|r| = \sum_i r_i$ and
$D^r = \frac{\partial^{|r|}}{\partial x_1^{r_1} \dots x_d^{r_d}}$.
The H\"older class $\Hcal_d(\beta,L)$ is the set of functions of $L_2(\Xcal)$
satisfying
\begin{equation}
|D^rf(x)-D^rf(y)| \leq L \|x-y\|^{\beta - |r|},
\end{equation}
for all $r$ such that $|r| \leq \floor{\beta}$ and for all $x,y\in\Xcal$.
\end{definition}

The following result establishes concentration of kernel density estimators.
At a high level, we follow the standard KDE analysis techniques to decompose
the $\Ltwo$ error into bias and variance terms and bound them separately.
A similar result for 2-dimensional densities was given by~\citet{liu11forest}.
Unlike the previous work, here we deal with the general $d$-dimensional case
as well as explicitly delineate the dependencies of the concentration bounds on
the deviation, $\varepsilon$.

\insertprespacing
\begin{lemma}
\label{lem:concKDE}
Let $f \in \Hcal_d(\beta, L)$ be a density on $[0,1]^d$ and assume we have $N$
\iid samples $\{X_i\}_{i=1}^N\sim f$.
Let $\fhat$ be the kernel density estimate obtained using a kernel with order
at least $\beta$ and bandwidth $h = \big(\frac{\logN}{N}\big)^{\ootbpd}$.
Then there exist constants $\kappa_1,\kappa_2,\kappa_3,\kappa_4 > 0$ such that
for all $\varepsilon < \kappa_4$ and number of samples satisfying
$\frac{N}{\logN} > \frac{\kappa_1}{\varepsilon^{\tpdob}}$ we have,
\begin{equation}
\PP\left( \|\fhat-f\|_\Ltwo > \varepsilon\right)
\leq \kappa_2\exp\left( -\kappa_3 N^{\tbotbpd}(\logN)^{\dotbpd} \varepsilon^2\right)
\end{equation}
\end{lemma}
\begin{proof}
First note that
\begin{align*}
\PP\big( \|\fhat-f\|_\Ltwo > \varepsilon\big) \leq
\PP\big( \|\fhat-\EE\fhat\|_\Ltwo + \|\EE\fhat-f\|_\Ltwo > \varepsilon\big).
\label{eqn:firstUnionBound}\numberthis
\end{align*}
Using the H\"olderian conditions and assumptions on the kernel,
standard techniques for analyzing the KDE~\citep{wasserman06nonparametric,
tsybakov08nonparametric}, give us a bound on the bias,
$\|\EE\fhat-f\|_\Ltwo \leq \kappa_5 h^\beta$, where
$\kappa_5 = L\int K(u)u^\beta \ud u$.
When the number of samples, $N$, satisfies
\begin{equation}\label{eqn:cond1}
\frac{N}{\logN} \;>\; \left(\frac{2\kappa'_5}{\varepsilon}\right)^{\tpdob} = \frac{\kappa_5}{\varepsilon^{\tpdob}},\, \mathrm{where}\ \kappa_5 \defeq (2\kappa'_5)^{\tpdob}
\end{equation}
we have $\|\EE\fhat-f\|_\Ltwo \leq \varepsilon/2$, and
hence~\eqref{eqn:firstUnionBound} turns into
$\PP\big( \|\fhat-f\|_\Ltwo > \varepsilon\big) \leq
\PP\big( \|\fhat-\EE\fhat\|_\Ltwo > \varepsilon/2\big)$.

The main challenge in bounding the first term is that we want the difference
to hold in $\Ltwo$. The standard techniques that bound the pointwise variance
would not be sufficient here.
To overcome the limitations, we use Corollary~2.2 from~\citet{gine02rates}.
Using their notation we have,
\begin{align*}
\sigma^2 \;\;&=\;
\sup_{t\in[0,1]^d} \VV_{X\sim f}\left[ \frac{1}{h^d}
  \kernel\left(\frac{X-t}{h}\right) \right]\\
&\leq\;  \sup_{t\in[0,1]^d}  \frac{1}{h^{2d}}\int
    \kernel^2\left(\frac{x-t}{h}\right) f(x)\ud x \\
&= \sup_{t\in[0,1]^d} \frac{1}{h^d}\int K^2(u)f(t+uh)\ud u \;
\leq \; \frac{\|f\|_\infty \|K\|_\Ltwo}{h^d} \\
U \;\;&=\; \sup_{t\in[0,1]^d}
  \left\|\frac{1}{h^d}\kernel\left(\frac{X-t}{h}\right) \right\|_\infty
  \;=\;\frac{\|K\|_\Linf}{h^d}.
\end{align*}
Then, there exist constants $\kappa_2,\kappa_3,\kappa'_4$ such that for all
$\varepsilon \in \left(\kappa'_4\frac{\sigma}{\sqrt{n}}
\sqrt{\log\frac{U}{\sigma}},\, \frac{\sigma^2}{U} \kappa'_4 \right)$
we have,
\begin{equation*}
\PP\left(\|\fhat-\EE\fhat\|_\Ltwo > \frac{\varepsilon}{2}\right)
\leq \kappa_2\exp\left( -\kappa_3 Nh^d \varepsilon^2\right).
\end{equation*}
Substituting for $h$ and then combining this with~\eqref{eqn:firstUnionBound}
gives us the probability inequality of the theorem. All that is left to do is
to verify the that the conditions on $\varepsilon$ hold.
The upper bound condition requires
$\varepsilon \leq
\frac{\kappa'_4\|f\|_\infty\|\kernel\|_\Ltwo}{\|\kernel\|_\Linf}
\defeq \kappa_4$.
After some algebra, the lower bound on $\varepsilon$ reduces to
$\frac{N}{\logN} \;>\; \frac{\kappa_6}{\varepsilon^{\tpdob}}$.
Combining this with the condtion~\eqref{eqn:cond1} and taking
$\kappa_1 = \max(\kappa_6, \kappa_5)$ gives the theorem.
\end{proof}

In order to apply the above lemma, we need $\Po,\Pto,\Ptto$ to satisfy the
H\"older condition. The following lemma shows that if all $\Obsk$'s are
H\"olderian, so are $\Po,\Pto,\Ptto$.

\insertprespacing
\begin{lemma}
\label{lem:PHolder}
Assume that the observation probabilities belong to the one dimensional
H\"older class; $\forall \ell\in[m], \Obsl\in\Hcal_1(\beta, L)$.
Then for some constants $L_1, L_2, L_3$,
$\Po \in \Hcal_1(\beta, L_1)$, $\Pto\in\Hcal_2(\beta, L_2)$,
$\Ptto\in\Hcal_3(\beta, L_3)$.
\end{lemma}
\begin{proof}
We prove the statement for $\Pto$. The other two follow via a similar argument.
Let $r = (r_1, r_2)$, $r_i \in \NN$, $|r| = r_1 + r_2 \leq \beta$, and
let $(s,t), (s',t') \in [0,1]^d$. Note that we can write,
\begin{equation*}
\Pto(s,t) = \sum_{k\in[m]}\sum_{\ell\in[m]} p(x_2=s, x_1=t, h_2=k, h_1=\ell)
= \sum_{k\in[m]}\sum_{\ell\in[m]} \alpha_{kl} \Obsk(s) \Obsl(t),
\end{equation*}
where $\sum_{k,\ell}\alpha_{k\ell} = 1$. Then,
\begingroup
\allowdisplaybreaks
\begin{align*}
&\frac{\partial^{|r|} \Pto(s,t)}{\partial s^{r_1} \partial t^{r_2}} -
\frac{\partial^{|r|} \Pto(s',t')}{\partial s^{r_1} \partial t^{r_2}} \\
&\hspace{0.4in}=\; \sum_{k,\ell} \alpha_{k\ell}
\left(\partialfrac{s^{r_1}}{\Obsk(s)} \partialfrac{t^{r_2}}{\Obsl(t)} \,-\,
\partialfrac{s^{r_1}}{\Obsk(s')} \partialfrac{t^{r_2}}{\Obsl(t')}\right) \\
&\hspace{0.4in}\leq\;
\sum_{k,\ell} \alpha_{k\ell} \left(
\left|\partialfrac{s^{r_1}}{\Obsk(s)}\right|
\left|\partialfrac{t^{r_2}}{\Obsl(t)}-\partialfrac{t^{r_2}}{\Obsl(t')}\right|
\;+\; \right. \\
&\hspace{1.5in}\left.
\left|\partialfrac{t^{r_2}}{\Obsl(t')}\right|
\left|\partialfrac{s^{r_1}}{\Obsk(s)}-\partialfrac{s^{r_1}}{\Obsl(s')}\right|
\right) \\
&\hspace{0.4in}\leq\;
\sum_{k,\ell} \alpha_{kl}\left( L'L|t-t'|^{\beta - r_2} +
L'L|s-s'|^{\beta - r_1} \right) & (\text{H\"older condition}) \\
&\hspace{0.4in}\leq\;
L'L\left(|t-t'|^{\beta - |r|} + |s-s'|^{\beta - |r|}\right) &
(\text{domain of } s, s' \text{ and } t, t') \\
&\hspace{0.4in}\leq\; L_2\sqrt{(t-t')^2 + (s-s')^2}^{\beta - |r|}
\end{align*}
Here, the third step uses the H\"older conditions on $\Obsk$ and $\Obsl$ and
the fact that the partial fractions are bounded in a bounded domain by a
constant, which we denoted $L'$, due to the H\"older condition.
Since $r_1 + r_2 = |r| \leq \beta$ and $r_1, r_2$ are positive integers,
we have $x^{\beta - r_i} \leq x^{\beta - r}, i = 1, 2$ for any $x \in [0, 1]$,
which implies the fourth step.
The last step uses Jensen's inequality and sets $L_2 \equiv L'L$.
\endgroup
% Note that we have effectively established a bound on the stronger metric $\|\fhat-f\|_\Linf$.
\end{proof}

The corollary belows follws as a direct consequence of Lemmas~\ref{lem:concKDE}
and~\ref{lem:PHolder}. We have absorbed the constants $L_1,L_2,L_3$ into
$\kappa_1,\kappa_2,\kappa_3,\kappa_4$.
\insertprespacing
\begin{corollary}
\label{cor:KDE}
Assume the HMM satisfies the conditions given in Section~\ref{sec:hmmintro}.
Let $\epslno,\epslnto,\epslntto\in(0,\kappa_4)$ and $\eta\in(0,1)$.
If the number of samples $N$ is large enough such that the following are true,
\begin{align*}
&\frac{N}{\logN} \,>\, \frac{\kappa_1}{\epslno^{2+\frac{1}{\beta}}}\;, \hspace{0.5in}
\frac{N}{\logN}  \,>\, \frac{\kappa_1}{\epslnto^{2+\frac{2}{\beta}}}\;, \hspace{0.5in}
\frac{N}{\logN}  \,>\, \frac{\kappa_1}{\epslntto^{2+\frac{3}{\beta}}}\;, \\
&N(\logN)^{\frac{1}{2\beta}} \,>\,
  \frac{1}{\epslno^{2+\frac{1}{\beta}}} \left(
\frac{1}{\kappa_3}\log\left(\frac{3\kappa_2}{\eta}\right) \right)^{1+\frac{1}{2\beta}} \\
&N(\logN)^{\frac{2}{2\beta}} \,>\,
  \frac{1}{\epslnto^{2+\frac{2}{\beta}}} \left(
\frac{1}{\kappa_3}\log\left(\frac{3\kappa_2}{\eta}\right) \right)^{1+\frac{2}{2\beta}} \\
&N(\logN)^{\frac{3}{2\beta}} \,>\,
  \frac{1}{\epslntto^{2+\frac{3}{\beta}}} \left(
\frac{1}{\kappa_3}\log\left(\frac{3\kappa_2}{\eta}\right) \right)^{1+\frac{3}{2\beta}}
\end{align*}
then with at least $1-\eta$ probability
the $\Ltwo$ errors between $\Po,\Pto,\Ptto$ and the KDE estimates
$\Pohat,\Ptohat,\Pttohat$ satisfy,
\begin{equation*}
\|\Po - \Pohat\|_\Ltwo \leq \epslno, \hspace{0.3in}
\|\Pto - \Ptohat\|_\Ltwo \leq \epslnto, \hspace{0.3in}
\|\Ptto - \Pttohat\|_\Ltwo \leq \epslntto.
\end{equation*}
\end{corollary}

% !TEX root = ../../paper.tex

\section{Analysis of the Spectral Algorithm}

Our proof is a brute force generalization of the analysis in~\citet{hsu09hmm}.
Following their template, we use establish a few technical lemmas.
We mainly focus on the cases where our analysis is different.

Throughout this section $\epslno,\epslnto,\epslntto$ will refer to $\Ltwo$
errors. Using our notation for c/q-matrices the errors can be written as,
\begin{align*}
\epslno &= \|\Po - \Pohat\|_\Ltwo = \|\Po-\Pohat\|_F, \\
\epslnto &= \|\Pto - \Ptohat\|_\Ltwo = \|\Pto-\Ptohat\|_F, \\
\epslntto &= \|\Ptto - \Pttohat\|_\Ltwo.
\end{align*}

We begin with a series of Lemmas.

\begin{lemma}
\label{lem:firstLemma}
Let $\epslnto \leq \varepsilon\sigmam(\Pto)$ where
$\varepsilon < \frac{1}{1+\sqrt{2}}$.
Denote $\varepsilon_0 = \frac{\epslnto^2}{( (1-\varepsilon)\sigmam(\Pto) )^2} < 1$.
Then the following hold,
\begin{enumerate}
\item $\sigmam(\Uhat^\top\Ptohat) \geq (1-\varepsilon)\sigmamPto$.
\item $\sigmam(\Uhat^\top\Pto) \geq \sqrt{1-\varepsilon_0}\sigmamPto$.
\item $\sigmam(\Uhat^\top\Pto) \geq \sqrt{1-\varepsilon_0}\sigmamPto$.
\end{enumerate}
\begin{proof}
The proof follows~\citet{hsu09hmm} after an application of
Weyl's theorem (Lemma~\ref{lem:weyl}) and
Wedin's sine theorem  (Lemma~\ref{lem:wedin}) for cmatrices.
\end{proof}
\end{lemma}

We define an alternative observable representation for the true HMM given by,
$\binftilde,\bonetilde\in\RR^m$ and $\Btilde:[0,1]\rightarrow \RR^{m\times m}$.
\begin{align*}
\bonetilde &= \Uhat^\top \Po = (\Uhat^\top\Obs)\pi \\
\binftilde &= (\Pto^\top\Uhat)\Po = (\Uhat^\top O)^{-1} \onem \\
\Bxtilde &= (\Uhat^\top \Ptxo)(\Uhat^\top\Pto)^\dagger = (\Uhat^\top \Obs)A(x)
   (\Uhat^\top \Obs)^{-1}.
\end{align*}
As long as $\Uhat^\top\Obs$ is invertible, the above parameters constitute a valid
observable representation. This is guaranteed if $\Uhat$ is sufficiently close to
$U$. We now define the following error terms,
\begin{align*}
\deltainf \;&=\quad \|\bUhtO^\top (\binfhat - \binftilde)\|_\infty =\quad
\|\bUhtO^\top\binfhat - \onem\|_\infty \\
\deltaone \;&=\quad\|\bUhtO^{-1}(\Bhat(x) - \Btilde(x))(\Uhat^\top\Obs)\|_1 =\quad
\|\bUhtO^{-1}\Bhat(x)\bUhtO - A(x) \|_1 \\
\Delta(x) \;&=\quad \|\bUhtO^{-1}(\Bxhat-\Bxtilde)\UhtO\|_1
  \;=\quad \|\bUhtO^{-1}\Bxhat - A(x)\|_1
\\
\Delta \;&=\quad \int_{x\in[0,1]} \Delta(x)\ud x
\end{align*}

The next lemma bounds the above quantities in terms of $\epslno,\epslnto,\epslntto$.

\begin{lemma}
\label{lem:secondLemma}
Assume $\epslnto < \sigmamPto/3$.
Then, there exists constants $c_1, c_2, c_3, c_4$ such that,
\begin{align*}
\deltainf &\;\leq\;\;
c_1\; \sigmaoneO \left( \frac{\epslnto}{\sigmamPto^2} \,+\,
    \frac{\epslno}{\sigmamPto}   \right)
\\
\deltaone &\;\leq\;\;
  c_2\;\frac{\epslno}{\sigmamO}
\\
\Delta(x) &\;\leq\;\;
  c_3\;\sqrt{m}\;\kappaO
  \left( \frac{\epslnto}{\sigmamPto^2}\|\Ptxo\|_2
  + \frac{ \| \Ptxo - \Ptxohat\|_2}{\sigmamPto^2}
  \right)
\\
\Delta &\;\leq\;\;
  c_4 \;\sqrt{m}\;\kappaO
  \left( \frac{\epslnto}{\sigmamPto^2}
  + \frac{ \epslntto}{\sigmamPto^2}
  \right)
\end{align*}
\begin{proof}
We will use $\lesssim,\gtrsim$ to denote inequalities ignoring constants.
First we bound $\deltainf \leq
\|\bUhtO^\top (\binfhat - \binftilde)\|_2 \leq \sigmaoneO
\|\binfhat - \binftilde\|_2$. Then we note,
\begingroup
\allowdisplaybreaks
\begin{align*}
\|\binfhat - \binftilde\|_2 &\leq
\|(\Ptohat^\top\Uhat)^\dagger\Pohat - (\Pto\Uhat)^\dagger\Po\|_2 \\
% \leq \|(\Pto^\top\Uhat)^\dagger -
&\leq \| (\Ptohat^\top\Uhat)^\dagger - (\Pto^\top\Uhat)^\dagger\|_2 \|\Pohat\|_2
  \;+\; \|(\Pto^\top\Uhat)^\dagger\|_2 \|\Ptohat-\Po\|_2 \\
&\lesssim \frac{\epslnto}{\min\{\sigmam(\Ptohat^\top), \sigmam(\Pto^\top\Uhat)\}^2}
 \;+\; \frac{\epslno}{\sigmam(\Pto^\top\Uhat)} \\
&\lesssim \frac{\epslnto}{\sigmam(\Pto)^2}
 \;+\; \frac{\epslno}{\sigmam(\Pto)}, \\
\end{align*}
\endgroup
where the third and fourth steps use Lemma~\ref{lem:firstLemma} and
Lemma~\ref{lem:pinv} (the pseudoinverse theorem for qmatrices).
This establishes the first result.
The second result is straightforward from Lemma~\ref{lem:firstLemma}.
\begin{equation*}
    \deltaone
    \leq \sqrt{m}\|\bUhtO^{-1}\|_2\|\bonehat - \bonetilde\|_2
    \leq \sqrt{m}\frac{\|\bonehat - \bonetilde\|_2}
      {\sigmam(\UhtO)}
    \lesssim \sqrt{m}\frac{\|\Uhat^\top(\Pohat-\Po)\|_2}
      {\sigmam(O)}
    \lesssim \frac{\sqrt{m}\epslno}{\sigmamO}.
\end{equation*}
For the third result, we first note
\begin{align*}
\Delta(x) &\leq \sqrt{m}\|\bUhtO^{-1}\|_2\|\Bxhat-\Bxtilde\|_2\|\UhtO\|_2
\leq \sqrt{m}\frac{\sigmaoneO}{\sigmam(\UhtO)}\|\Bxhat-\Bxtilde\|_2
\\
&\lesssim \sqrt{m}\;\kappaO\|\Bxhat-\Bxtilde\|_2
\end{align*}
To bound the last term we decompose it as follows.
\begin{align*}
\|\Bxhat-\Bxtilde\|_2
\;&=\; \|(\Uhat^\top\Ptxo)(\Uhat^\top\Pto)^\dagger -
  (\Uhat^\top\Ptxohat)(\Uhat^\top\Ptohat)^\dagger\|_2
\\
&\leq\;\|(\Uhat^\top\Ptxo)((\Uhat^\top\Pto)^\dagger-(\Uhat^\top\Ptohat)^\dagger)\|_2
\;+\;
\|\Uhat^\top(\Ptxo-\Ptxohat)(\Uhat^\top\Ptohat)^\dagger\|_2\\
&\leq
\|\Ptxo\|_2 \| \|(\Uhat^\top\Pto)^\dagger-(\Uhat^\top\Ptohat)^\dagger\|_2
\;+\; \|\Ptxo-\Ptxohat\|_2 \|(\Uhat^\top\Ptohat)^\dagger\|_2 \\
&\lesssim \|\Ptxo\|_2 \frac{\epslnto}{\sigmamPto^2} +
\frac{\|\Ptxo-\Ptxohat\|_2}{\sigmamPto}.
\end{align*}
This proves the third claim.
For the last claim, we make use of the proven statements. Observe,
\begin{align*}
\int \|\Ptxo\|_2\ud x
\;\leq\; \left( \int \|\Ptxo\|_2^2\ud x \right)^{1/2}
\;\leq\; \left( \int\int\int \Ptto(s,x,t)^2 \ud s \ud t\ud x
\right)^{1/2}
= \|\Ptto\|_\Ltwo,
\end{align*}
where the first step uses inclusion of the $L^p$ norms in $[0,1]$. The second
step uses $\|\cdot\|_2\leq \|\cdot\|_\frob$ for cmatrices.
A similar argument shows $\int_x \|\Ptxo-\Ptxohat\|_2 \leq \epslntto$.
Combining these results gives the fourth claim.
\end{proof}
\end{lemma}

Finally, we need the following Lemma. The proof almost exactly replicates the
proof of Lemma 12 in~\citet{hsu09hmm}, as all operations can be done
with just matrices.
\insertprespacing
\begin{lemma}
\label{lem:thirdLemma}
Assume $\epslntto \leq \sigmam(\Pto)/3$. Then $\forall t \geq 0$,
\begin{equation}
    % \|\potot - \phatotot\|_\Lone =
    \int |p(\xotot) - \phat(\xotot)|\ud \xotot
    \;\;\leq\quad
    \deltainf + (1+\deltainf)
    \left( (1+\Delta)^t\deltaone + (1+\Delta)^t - 1\right),
\end{equation}
where the integral is over $[0,1]^t$.
\end{lemma}

We are now ready to prove Theorem~\ref{thm:jointProbThm}.

\begin{proof}[\textbf{Proof of Theorem~\ref{thm:jointProbThm}}]
If $\epslno,\epslnto,\epslntto$ satisfy the following for appropriate
choices of $c_5,c_6,c_7$,
\begin{align*}
% \hspace{-0.05in}
\epslno \;\leq\; c_5 \min(\sigmamPto, \frac{\kappaO}{\sqrt{m}}) \epsilon,
\quad
\epslnto \;\leq\; c_6 \frac{\sigmamPto^2}{\kappaO} \epsilon,
\quad
\epslntto \;\leq\; c_7 \frac{\sigmamPto}{\sigmaoneO} \frac{1}{t\sqrt{m}}\epsilon,
\numberthis
\label{eqn:epslnBounds}
\end{align*}
we then have $\deltaone \leq \epsilon/20$, $\deltainf\leq \epsilon/20$
and $\Delta \leq 0.4\epsilon/t$.
Plugging these expressions into Lemma~\ref{lem:thirdLemma} gives
$\int |p(\xotot) - \phat(\xotot)|\ud \xotot \leq \epsilon$.
When we plug the expresssions for $\epslno,\epslnto,\epslntto$
in~\eqref{eqn:epslnBounds} into
Corollary~\ref{cor:KDE} we get the required sample complexity.
\end{proof}

% !TEX root = ../../paper.tex

\section{Addendum to Experiments}
\label{app:experiments}

\textbf{Details on Synthetic Experiments:}
Figure~\ref{fig:emissions} shows the emission probabilities used in our synthetic
experiments. For the transition matrices, we sampled the entries of the matrix from
a $U(0,1)$ distribution and then renormalised the columns to sum to $1$.

In our implementation, we use a Gaussian kernel for the KDE which is of order
$\beta=2$. While higher order kernels can be constructed using Legendre
polynomials~\citep{tsybakov08nonparametric}, the Gaussian kernel was more robust in
practice. The bandwidth for the kernel was chosen via cross validation on density
estimation.

\insertToyPDFs

\textbf{Details on Real Datasets:}
Here, we first estimate the model parameters using the training sequence.
Given a test sequence $x_{1:n}$, we predict $x_{t+1}$ conditioned on the previous
$x_{1:t}$ for $t=1:n$.
\begin{enumerate}
\item Internet Traffic.
Training sequence length: $10,000$. Test sequence length: $10$.
\item Laser Generation.
Training sequence length: $10,000$. Test sequence length: $100$.
\item Physiological data.
Training sequence length: $15,000$. Test sequence length: $100$.
\end{enumerate}

\end{document}